\documentclass[10pt]{article}

\usepackage[utf8]{inputenc} 
\usepackage[T1]{fontenc}    
\usepackage{hyperref}       
\usepackage{url}            
\usepackage{booktabs}       
\usepackage{amsfonts}       
\usepackage{nicefrac}       
\usepackage{microtype}      
\usepackage{amsmath, amsthm, amssymb, multirow, paralist}
\usepackage[numbers]{natbib}
\usepackage[noend]{algpseudocode}
\usepackage{dsfont}
\usepackage{wrapfig}
\usepackage{graphicx}
\graphicspath{{figure/}}
\usepackage{subfigure}
\usepackage{comment}
\usepackage{multirow}
\usepackage{makecell}
\usepackage{bm}
\usepackage{setspace}
\usepackage{stmaryrd}
\usepackage{enumerate}
\usepackage{enumitem}
\newtheorem{theorem}{Theorem}

\newtheorem{definition}{Definition}

\newtheorem{corollary}[theorem]{Corollary}
\newcommand{\argmin}{\mathop{\arg\min}}
\newcommand{\argmax}{\mathop{\arg\max}}
\usepackage[noend]{algpseudocode}
\usepackage{algorithmicx,algorithm}

\usepackage{float}

\def \X {\mathcal{X}}
\def \Y {\mathcal{Y}}

\def \R {\mathcal{R}}
\def \S {\mathcal{S}}
\def \F {\mathcal{F}}
\def \G {\mathcal{G}}

\def \RR {\mathbb{R}}
\def \EE {\mathbb{E}}
\def \II {\mathbb{I}}
\def \Ra {\mathfrak{R}}

\usepackage[body={6.5in,8in}]{geometry}

\date{}

\author{
	Jiaqi~Lv$^1$, Biao~Liu$^2$, Lei~Feng$^3$, Ning~Xu$^2$, Miao~Xu$^4$, Bo~An$^5$, Gang~Niu$^1$,\\ Xin~Geng$^2$, Masashi~Sugiyama$^{1,6}$\\
	$^1$RIKEN Center for Advanced Intelligence Project\\
	$^2$School of Computer Science and Engineering, Southeast University\\
	$^3$College of Computer Science, Chongqing University\\
	$^4$School of Information Technology and Electrical Engineering, The University of Queensland\\
	$^5$School of Computer Science and Engineering, Nanyang Technological University\\
	$^6$Graduate School of Frontier Sciences, The University of Tokyo\\
	\texttt{\{is.jiaqi.lv, gang.niu.ml\}@gmail.com}, \\
	\texttt{\{liubiao01, xning, xgeng\}@seu.edu.cn}, \\
	\texttt{lfeng@cqu.edu.cn}, \texttt{miao.xu@uq.edu.au}\\
	\texttt{boan@ntu.edu.sg}, \texttt{sugi@k.u-tokyo.ac.jp}
}

\title{On the Robustness of Average Losses for Partial-Label Learning}

\begin{document}

\maketitle

\begin{abstract}
\emph{Partial-label learning}~(PLL), as a typical weakly supervised learning problem, trains multi-class classifiers from instances with partial labels---a partial label for an instance is a set of candidate labels where a \emph{fixed but unknown} candidate is the true label.
There are two mainstream approaches to PLL: (a) the \emph{identification-based strategy}~(IBS) purifies each partial label on the fly to select the (most likely) true label for training; (b) the \emph{average-based strategy}~(ABS) treats all candidate labels equally for training and let trained models be able to predict the true label of any instance.
The research of PLL has focused on IBS due to its better performance.
However, we argue that ABS is also worthy of study, since it follows \emph{empirical risk minimization} and thus it is easier to analyze; more importantly, \emph{all modern IBS methods behave like ABS in the beginning of training} to prepare for partial-label purification and true-label selection.
In this paper, we analyze why the performance of ABS was unsatisfactory and propose how to improve it theoretically and practically.
Specifically, we first formalize five problem settings for the generation processes of noise-free and noisy partial labels, and then prove that \emph{average partial-label} (APL) losses with \emph{bounded} multi-class losses are \emph{always} robust under mild assumptions, while APL losses with \emph{unbounded} multi-class losses (e.g., the cross-entropy loss) may not be robust.
Given that there exists no such analysis for IBS yet, our robustness analysis is novel for not only ABS but also PLL.
We have two promising experimental findings: (a) ABS methods using bounded losses can match or even exceed the state-of-the-art performance of IBS methods using unbounded losses; (b) after using robust APL losses to warm start, IBS methods can further improve upon themselves.
Considering the theoretical superiority and practical potential of ABS, our work draws attention to the design of more advanced ABS methods, which can in turn boost IBS and push forward PLL as a whole.
\end{abstract}

\section{Introduction}

Deep neural networks (DNNs) have become the par excellence base model in diverse application domains, which transform the input data (e.g., images) to the specific outputs (e.g., classes). 
Much of the success in running DNNs is attributed to its internal capability to approximate arbitrarily complex functions mapping input to output \cite{cybenko1989approximation,anthony2009neural,du2019gradient,allen2019convergence}, as well as an external driving force---labeled training data.
It is widely believed that the performance of DNNs is improved as the number of data increases, reaching saturation only when millions of data are available \cite{zhou2017places,sun2017revisiting,mahajan2018exploring,radford2019language}.
Their remarkable performance usually comes at a prohibitively high labeling cost, especially when data labeling must be carried out professionally.
A shortage of skilled experts, an expensive and time-consuming labeling process, and privacy issues can pose challenges to the acquisition of high-quality labels.
As a result, learning with imperfect but inexpensive labels is practically significant. 

\begin{figure*}[!t]
	\begin{minipage}{.04\textwidth}
		\centering
		\begin{tabular}{l}
			MNIST
		\end{tabular}
	\end{minipage}
	\hfill
	\begin{minipage}{.26\textwidth}
		\centering
		\includegraphics[width=4cm]{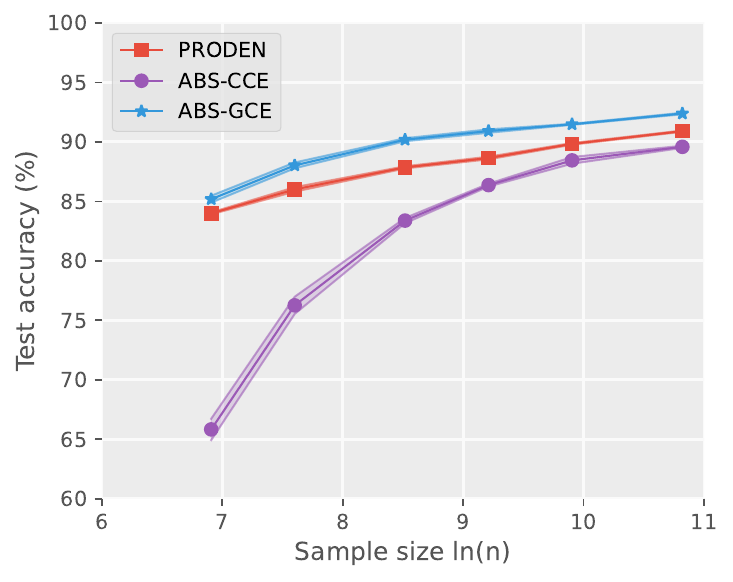}
		\centerline{\quad Linear with NF-PLs}
	\end{minipage}\hspace{-8mm} 
	\begin{minipage}{.26\textwidth}
		\centering
		\includegraphics[width=4cm]{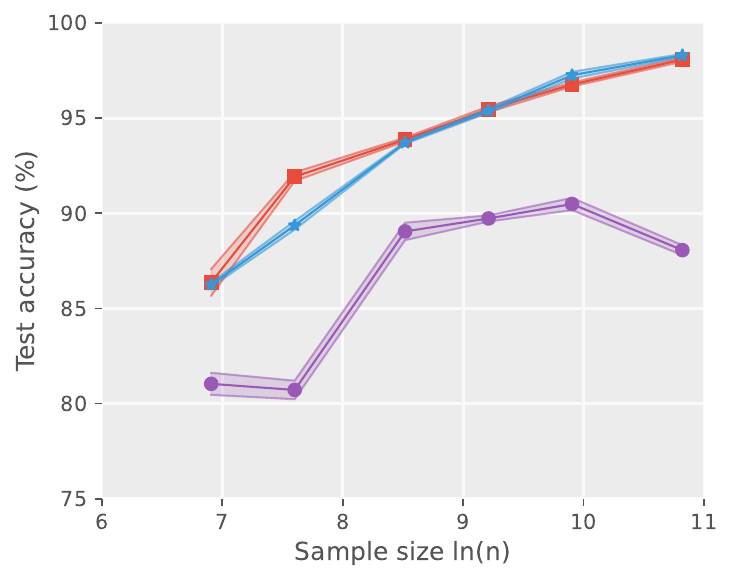}
		\centerline{\quad MLP with NF-PLs}
	\end{minipage}\hspace{-8mm}
	\begin{minipage}{.26\textwidth}
		\centering
		\includegraphics[width=4cm]{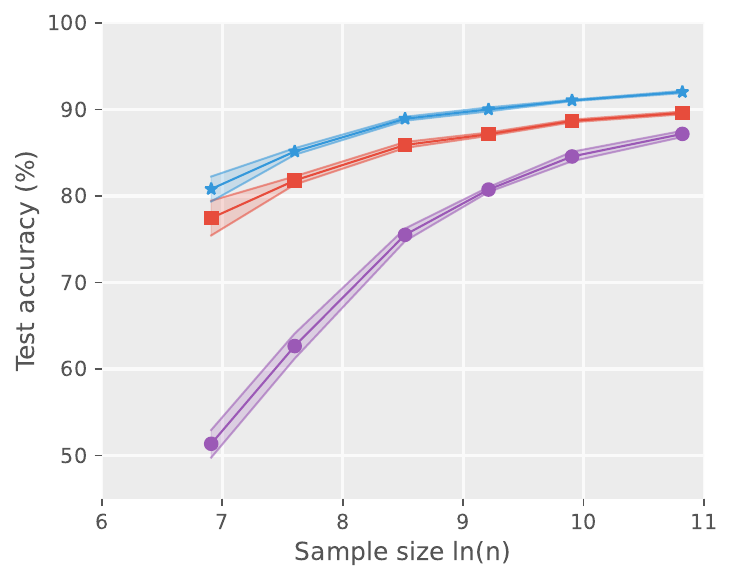}
		\centerline{\quad Linear with N-PLs}
	\end{minipage}\hspace{-8mm}
	\begin{minipage}{.26\textwidth}
		\centering
		\includegraphics[width=4cm]{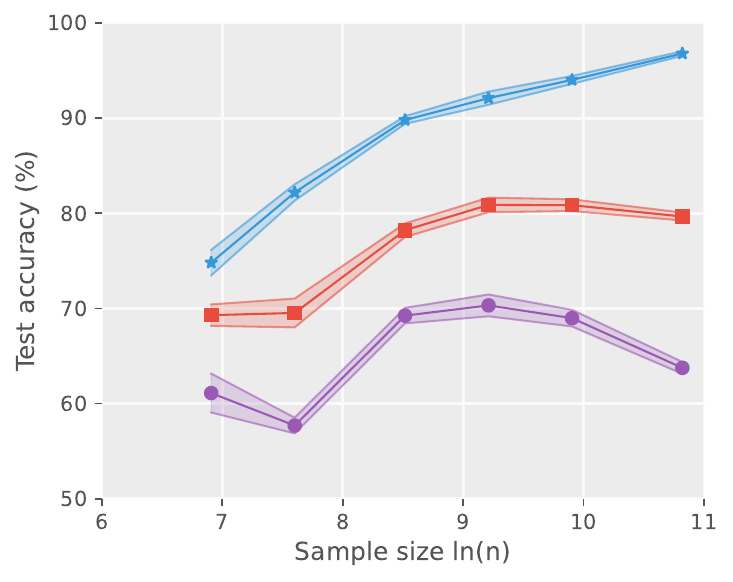}
		\centerline{\quad MLP with N-PLs}
	\end{minipage}
	
	\begin{minipage}{.04\textwidth}
		\centering
		\begin{tabular}{l}
			CIFAR-\\10
		\end{tabular}
	\end{minipage}
	\hfill
	\begin{minipage}{.26\textwidth}
		\centering
		\includegraphics[width=4cm]{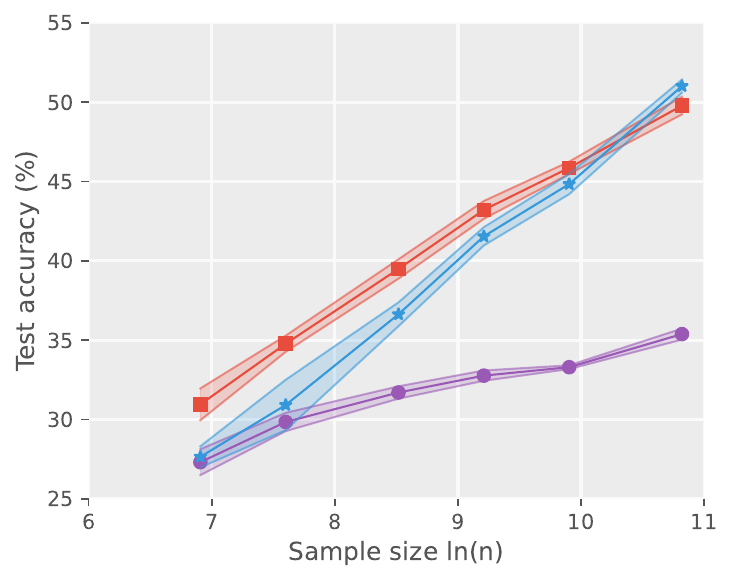}
		\centerline{\quad MLP with NF-PLs}
	\end{minipage}\hspace{-9mm} 
	\begin{minipage}{.26\textwidth}
		\centering
		\includegraphics[width=4cm]{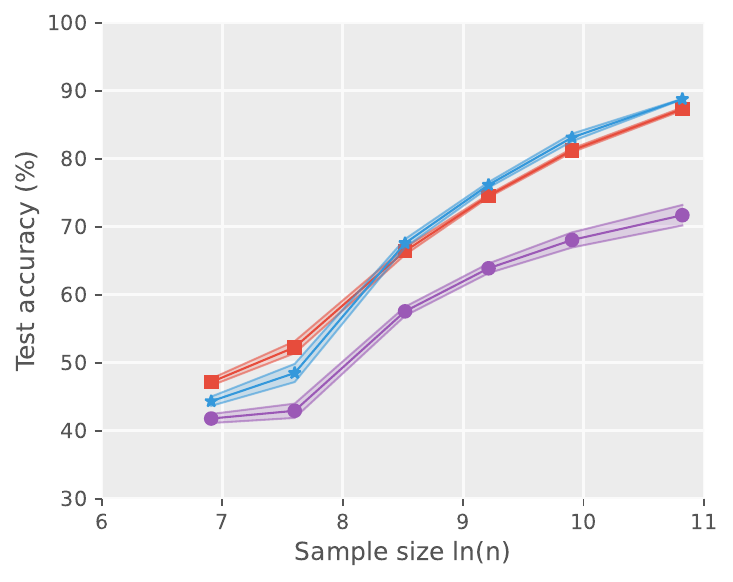}
		\centerline{\quad ConvNet with NF-PLs}
	\end{minipage}\hspace{-9mm} 
	\begin{minipage}{.26\textwidth}
		\centering
		\includegraphics[width=4cm]{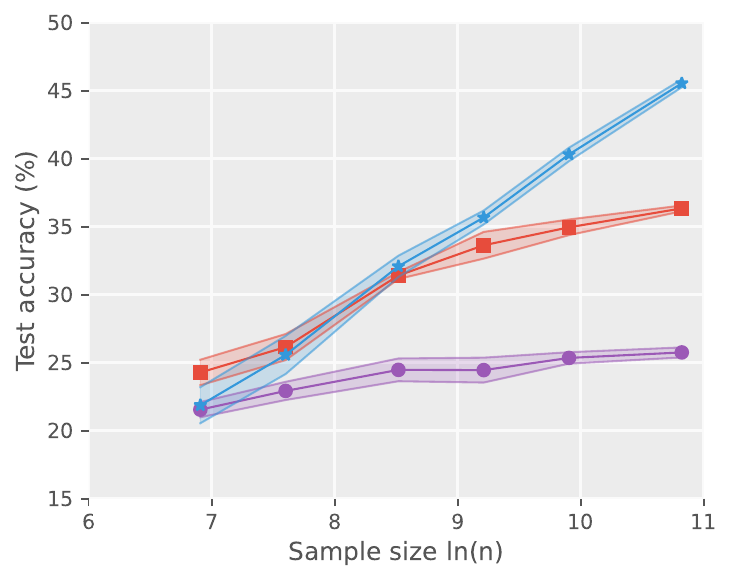}
		\centerline{\quad MLP with N-PLs}
	\end{minipage}\hspace{-9mm} 
	\begin{minipage}{.26\textwidth}
		\centering
		\includegraphics[width=4cm]{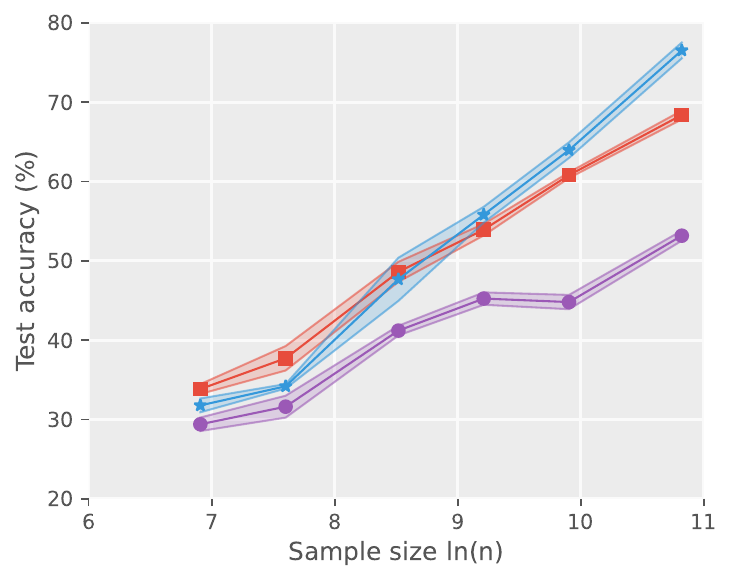}
		\centerline{\quad ConvNet with N-PLs}
	\end{minipage}
	\caption{Comparison of a SOTA IBS method PRODEN \cite{lv2020progressive} and our ABS method with CCE or GCE for PLL under different training dataset sizes. 	
		All experiments were repeated 5 times and the standard deviations are presented in shadow.
		In the upper row, on MNIST, a linear-in-input model (Linear) and a 5-layer multi-layer perceptron (MLP) were trained by stochastic gradient descent \cite{robbins1951stochastic}. 
		In the bottom row, on CIFAR-10, an MLP-5 and a 12-layer convolutional neural network (ConvNet) \cite{laine17temporal} were trained by Adam \cite{diederik2015adam}. 
		We can clearly see our ABS method with GCE matches the performance of PRODEN with noise-free (NF-) PLs, and exceeds PRODEN with noisy (N-) PLs, where the true label might not be included in PLs.}
	\label{fig:ibsVSabs}
\end{figure*}

Crowdsourcing \cite{brabham2008crowdsourcing} relying on non-expert workers has recently emerged as an attractive surrogate.
Unlabeled instances are typically assigned to workers of varying knowledge, and limited by their expertise, they often have difficulty recognizing the exact label from multiple ambiguous categories.
Therefore, crowdsourcing platforms naturally allow workers to select several possible labels if they are uncertain about an instance.
In this way, an instance is associated with a set of candidate labels where a \emph{fixed but unknown} candidate is the true label. 
A set of candidate labels is referred to as a \emph{partial label} (PL) for an instance, and the learning paradigm that can handle PLs is termed as \emph{partial-label learning} (PLL) \cite{nguyen2008classification,cour2011learning,zhang2016par,zhang2017disambiguation,wu2018towards,xu2019partial,cabannes202structured,yao2020network,lv2020progressive,feng2020provably,lyu2021gmpll,wen2021leveraged}, also known as ambiguous label learning \cite{chen13dictionary,chen2017learning,gong2018regularization,yao2020deep} and superset learning \cite{hullermeier2015superset,liu2014learnability,liu2012conditional}.
PLL attempts to infer the optimal multi-class classifier that is able to accurately predict the true label for unseen instances by fitting PLs, and more ideally, the hypotheses can be modeled by DNNs. 
PLL problems arise in real-world scenarios \cite{luo2010,liu2012conditional,zeng13learning} as well.

Research on PLL dates back about 20 years. 
Initially, Jin and Ghahramani \cite{jin2003learning} built up a maximum likelihood model to reassign class-posterior probabilities to candidate labels iteratively.
This work opened up a main research route of PLL that purifies each PL on the fly to select the most likely true label during training \cite{jin2003learning,chen13dictionary,feng2019partialb,feng2019partial,xu2019partial,wen2021leveraged}, which is named the \emph{identification-based strategy} (IBS).
Because IBS aims at eliminating the ambiguity \cite{zhang2017disambiguation} between individual instances and their true labels in the training phase, this technique is also commonly known as \emph{disambiguating} \cite{jin2003learning,cour2011learning}.
Contrariwise, H{\"{u}}llermeier and Beringer \cite{hullermeier2006learning} formalized PLL as a collaborative problem, where all candidate labels contribute to the learning objective \emph{equally}.
The idea is that the inductive bias underlying the learning process can benefit disambiguating the given PLs, and let trained models be able to predict the true label of any instance.
Such a scheme is called the \emph{average-based strategy} (ABS) \cite{cour2011learning,zhang2015solving}.

In recent years, the research of PLL has focused on IBS, since it was believed that the performance of IBS is more promising, while little attention has been paid to ABS.
This is especially so in the era of deep learning \cite{lv2020progressive,feng2020provably,wen2021leveraged}.
The pessimism about ABS comes from ``memorization'' of over-parameterized DNNs, that is, the perfectly fitted DNNs can memorize all training samples, even if their labels are completely arbitrary \cite{zhang2017understanding,feldman2020neural}.
ABS is free from identifying the latent true labels during training, and therefore memorize all candidate labels.
Then the PLL problem would be degenerated to a multiple-label problem \cite{jin2003learning}, where it is acceptable that an arbitrary candidate label is taken for the ``pseudo true label''.
Then will ABS fail in the true-label prediction and should ABS just be phased out by the times?

In this paper, we argue the research value of ABS by showing its practical potential and theoretical superiority, problematizing the traditional view of ABS and pushing forward PLL as a whole.

\vspace{1em}  
\noindent
\textbf{Promising experimental findings} 
\quad Our work is inspired by a set of exploratory experiments.
We propose a family of ABS losses named \emph{average partial-label} (APL) losses, which are defined as the average of multi-class losses over all candidate labels.
The categorical cross entropy (CCE) loss is the most popular multi-class loss in deep learning nowadays, and we observed that all existing deep IBS methods \cite{lv2020progressive,feng2020provably,wen2021leveraged,yao2020deep} also adopted the CCE loss.
Thus firstly, we trained several standard deep models with the APL loss equipped with the CCE loss on benchmark datasets where true labels were manually corrupted to PLs. 
In this case, we found that our ABS method with the CCE loss performs poorly as was previously thought, shown in Figure~\ref{fig:ibsVSabs}.
Extending these experiments, we then replaced the CCE loss with another widely-used loss, i.e., the generalized cross entropy (GCE) loss \cite{zhang2018generalized}, and conducted experiments with early stopping \cite{morgan89generalization}.
Surprisingly, we observed that our ABS method with the GCE loss can often be on a par with, or even outperform a SOTA IBS method PRODEN \cite{lv2020progressive}, regardless of datasets, models, and optimizers.
For the details of data generation processes, please see Section~\ref{sec:empirical}.
These observations challenge common beliefs since one would expect that ABS is incapable of distinguishing true labels, thereby hurting generalization, while the results showed that ABS method with the APL loss can also predict the true label.
Therefore, we would like to analyze \emph{what is it that distinguishes ABS methods that perform well from those that do not}, and the answer to this question will hopefully help improve ABS.

\vspace{1em}  
\noindent
\textbf{Novel robustness analysis} 
\quad To answer this question, we analyze the \emph{robustness} of our ABS method to PLs, namely, whether the classification error on supervised data of the minimizer of the risk w.r.t. the APL losses is approximated to that of the Bayes classifier (learned using supervised data) \cite{manwani2013noise,ghosh2015making,ghosh2017robust,menon2019can}.
Thanks to the concise form of the APL losses, it is easy to estimate the risk under the APL losses from PLs and carry out \emph{empirical risk minimization} (ERM).
Thus we can analyze the robustness through existing mathematical techniques.

Furthermore, we take one more step forward---\emph{unreliable PLL}, which learns from \emph{noisy PLs}, that is, the candidate-label set that \emph{might not} include the true label.
As the acquisition of training data expands, noise is inevitable, therefore, it is not an overstatement to say unreliable PLL is imminent in real-world applications. 
Unfortunately, previous PLL algorithms concentrate on noise-free PLs, and they have not been able to handle the unreliable PLL well, as shown in the two rightmost columns in Figure~\ref{fig:ibsVSabs}.
To avoid confusion, we refer to the traditional PLL paradigm \emph{reliable PLL}, to reliable PLL and unreliable PLL collectively as ``PLL'', and to noise-free PL and noisy PL collectively as ``PL'' in later sections.

To theoretical analyze the cause of success or failure of an ABS method, we formalize five problem settings for the generation processes, two of which are for noise-free PLs and three are for noisy PLs.
With the help of them, we delve into multiple widely-used multi-class loss functions, and formally prove that APL losses with \emph{bounded} loss functions (e.g., GCE) are always robust under mild assumptions on the domination of true labels, while APL losses with \emph{unbounded} loss functions (e.g., CCE) may not be robust.
The theoretical results are reconciled with experimental observations in Figure~\ref{fig:ibsVSabs}.
Given that there exists no such analysis for IBS yet, our robustness analysis is novel for not only ABS but also PLL.

\vspace{1em}  
\noindent
\textbf{ABS improvement to IBS} 
\quad Moreover, we rethink the existing deep IBS methods.
We point out that \emph{all modern IBS methods behave like ABS in the beginning of training} to prepare for PL purification and true-label selection.
In other words, they need to use ABS to warm start model training, and use the pretrained model to identify the true label.
Thus they will select the correct true labels if ABS can become better, while they have hitherto used unbounded losses throughout the training process.
As a consequence, IBS methods can in turn improve upon themselves by our study on ABS: we suggest utilizing bounded losses as a warmup for the first few epochs.
We conduct extensive experiments to verify the effectiveness of this improvement.

\vspace{1em}  
\noindent
\textbf{Contributions} 
\quad Our contributions can be summarized:
\begin{itemize}[topsep=0ex,itemsep=0ex,leftmargin=*]
	\item We establish a theoretically grounded framework for ABS based on a simple yet effective APL loss family, the risk minimization of which is guaranteed to be robust to PLs under five problem settings for the data generation processes.
	\item To the best of our knowledge, we are the first to propose the unreliable PLL paradigm, further developing the practical potentiality of PLL in society. 
	Our ABS method with the APL losses provides an effective baseline for unreliable PLL, and it also works well for reliable PLL without any modifications.
	\item We redraw the attention of the PLL community to ABS. 
	Our research findings can not only improve ABS, but also enlighten a general principle to incorporate ABS into IBS methods to further enhance their performance, and push forward PLL as a whole.
\end{itemize}

\vspace{1em}  
\noindent
\textbf{Paper organization} 
\quad We recapitulate the related work and discuss the philosophies behind ABS and IBS in Section~\ref{sec:related}.
In Section~\ref{sec:pre}, we give an overview of the problem setting and introduce the importance of robustness analysis.
In Section~\ref{sec:method}, we propose APL losses and formalize five generation processes of PLs.
We present our main theoretical results and experimental findings in Section~\ref{sec:result} and Section~\ref{sec:exp}, respectively.
We conclude in Section~\ref{sec:conclu}, and defer additional experimental results and all the proofs to the appendix.

\section{Related Work}\label{sec:related}

In this section, we review some seminal work in reliable PLL, discuss philosophies behind different technical routes of PLL, and analyze its relation to and difference from other machine learning problems.

\vspace{1em}
\noindent
\textbf{Practical reliable PLL} 
\quad \cite{jin2003learning} is one of the milestones for reliable PLL.
It proposed to disambiguate noise-free PLs by using the expectation-maximization (EM) algorithm. 
In the E-step, the class-posterior probability is estimated as the normalization of current model predictions.
In the M-step, the model parameters are updated in order to minimize the KL divergence between the given estimated
probabilities and the model-based distributions.
Such a strategy that identifies true labels along with model training is referred to as IBS.
Following the milestone work, many IBS algorithms have been developed (e.g., \cite{chen13dictionary,zhang2015solving,zhang2016par,gong2018regularization,xu2019partial,lyu2021gmpll,feng2019partial,tang2017confidence}).

Reliable PLL also has been studied along the other research route called ABS, pioneered by H{\"{u}}llermeier and Beringer \cite{hullermeier2006learning}.
They determined the class label for an unseen instance by voting among the candidate labels of its nearest neighbors.
Cour \textit{et al.} \cite{cour2011learning} proposed a convex loss that distinguishes the averaged output over the candidate labels from outputs over non-candidate labels.
It is always believed that ABS is likely to fail as the outputs of pseudo true labels would overwhelm the output of true label.
Therefore, the development of the two strategies was not balanced -- IBS has been the focus of considerable recent research whilst ABS faded away.

From 2020, deep learning starts injecting new vitality into IBS.
Almost at the same time, three works proposed to model classifiers by DNNs.  
Yao \textit{et al.} \cite{yao2020deep} adopted ResNet as the backbone together with two specially designed regularizers for partially-labeled image classification.
Yao \textit{et al.} \cite{yao2020network} used the co-teaching scheme \cite{han2018co} to let two networks interact with each other regarding the confidence levels of the instances. 
The method proposed by Lv \textit{et al.} \cite{lv2020progressive} progressively identifies true labels based on the memorization effect of DNNs \cite{arpit2017closer}, which is flexible on the learning models and loss functions.
Later, Feng \textit{et al.} \cite{feng2020provably} formalized for the first time the generation process of noise-free PLs, based on which they derived two provably consistent algorithms. 
Wen \textit{et al.} \cite{wen2021leveraged} proposed a leveraged weighted loss to trade off the losses on candidate labels and non-candidate ones.
Wu and Sugiyama \cite{wu2021learning} proposed a unified framework includes \cite{feng2020provably} as a special case. 
These three works are both compatible with DNNs. 

\vspace{1em}                         
\noindent
\textbf{Provable reliable PLL} 
\quad 
Although the above practical algorithms have proven empirically successful on specific domains, there is an elusive theoretical gap in the understanding of them. 
Through the lens of learning theory, some researchers proposed seminal theoretical works in reliable PLL.
Liu and Dietterich \cite{liu2014learnability} proposed the \emph{small ambiguity degree condition} to ensure that classification errors on any instance have a probability of being detected.
The proof of this theorem requires strict assumptions: the approximation error equals zero and meanwhile the Bayes error equals zero (i.e., the deterministic scenario \cite{mohri2018foundations}).
Cour \textit{et al.} \cite{cour2011learning}, Feng \textit{\textit{et al.}} \cite{feng2020provably}, and Wen \textit{\textit{et al.}} \cite{wen2021leveraged} focused on the statistical consistency.
They proposed a consistent loss based on some specific data generation process or deterministic scenario assumption, while our findings are general enough to hold under different generation processes and also a stochastic scenario \cite{mohri2018foundations}. 

\vspace{1em}               
\noindent
\textbf{Philosophies behind IBS and ABS} 
\quad IBS iterates between the optimization of a learning model and the identification of the true label.
Typically the identified true label has the biggest posterior of all labels, and must be in the candidate-label set.
In other words, the ``true'' one is the ``ideal'' one.
It implies that the true label can be \emph{uniquely determined} given an input---it is satisfied only in the \emph{deterministic scenario} where the class-posterior probability of the true label is equal to 1.

However, the natural world is more like the \emph{stochastic scenario} that possesses some inherent randomness. 
In this setting, the label is a probabilistic function of the input, indicating that the same input will lead to an ensemble of \emph{unfixed} output labels.
ABS essentially gets rid of the deterministic scenario by avoiding recognizing the ``ideal'' label.
The ``true'' label is considered as an ``actually sampled'' outcome, and consequently, the philosophy of ABS is compatible with the stochastic scenario.
Therefore, it is crucial to design advanced methods and provide theoretical understandings for ABS.

\vspace{1em}                         
\noindent
\textbf{Relevant learning problems} 
\quad There are some weakly supervised learning problems related to PLL.

\begin{itemize}[topsep=0ex,itemsep=0ex,leftmargin=*]
	\item \emph{Complementary-label (CL) learning} \cite{ishida2017learning,yu2018learning,ishida2019complementary,gao2021discriminative} learns from weakly-supervised datasets wherein an instance is equipped with a CL. 
	A CL specifies a class that the pattern does NOT belong to, so it can be considered as an extreme noise-free PL case with a fixed number ($k-1$) of candidate labels.
	Then from the algorithmic point, reliable PLL algorithms can directly handle the CL learning problem, but not also the other way around.
	
	\item \emph{Semi-supervised learning (SSL)} \cite{zhu2009introduction,tarvainen2017mean,zhang2018mixup,miyato2018virtual,berthelot2019mixmatch} learns from datasets consisting of both labeled and unlabeled data.
	Since we can regard the universe set of labels as the candidate labels of unlabeled data, SSL has some relation with PLL.
	However, standard SSL assumes that labeled data are fully supervised, which is different from reliable PLL, where labeled data are still ambiguous.
	
	\item \emph{Noisy-label learning (NLL)} \cite{natarajan2013learning,liu2015classification,patrini2017making,han2018co,han2020survey} learns from noisy supervision where the training data are sampled from a corrupted distribution.
	Both NLL and PLL should have an underlying transition matrix linking the clean class posterior and the observed class posterior of an instance.
	Nonetheless, their matrix dimensions are different: the transition matrix is $k\times k$ in NLL and $k\times (2^k-2)$ in PLL.
	
\end{itemize}

\section{Preliminaries}\label{sec:pre}

In this section, we formally introduce reliable PLL and propose unreliable PLL, and give the definition of robustness.

\subsection{Problem Setup}

\textbf{Basic settings} 
\quad Let us consider a multi-class classification problem of $k$ classes. 
Let $\X\subseteq\RR^d$ be the feature space, $\Y=[k]\doteq\{1,2,\ldots,k\}$ be the label space, and $\S\doteq\{2^{[k]}\backslash\emptyset\backslash\Y\}$ be the \emph{PL space}. 
$2^{[k]}$ means the collection of all subsets in $[k]$, and $|\S|=2^k-2$ because the empty set and the whole label set are excluded.
We denote by $p(\boldsymbol{x},y)$ some probability density of ``clean'' distribution over $\X\times\Y$.
In fully-supervised classification, the goal is a learning model (e.g., a DNN) $f:\RR^d\rightarrow [k]$ that can make correct prediction on unseen inputs, with a set of i.i.d.~supervised training data $\{(\boldsymbol{x}_i,y_i)\}^n_{i=1}$ sampled from $p(\boldsymbol{x},y)$.
A classifier $f(\boldsymbol{x})$ is routinely assumed to take the following form:
\begin{equation*}
	f(\boldsymbol{x})=\argmax_{i\in \Y}g_i(\boldsymbol{x}),
\end{equation*}
where $g_i(\boldsymbol{x}):\RR^d\rightarrow\RR$ outputs a score for class $i$.
In this paper, we concentrate on deep learning: assume the learning model $f$ is a DNN and apply softmax operation to convert scores into a vector of class-posterior probabilities, i.e., $g_i(\boldsymbol{x})=p(i|\boldsymbol{x})\in\Delta^{k-1}$ \cite{raghu2017expressive}, where $\Delta^{k-1}$ denotes the $k$-dimensional simplex. 

While in PLL, for the notional clean distribution with probability density $p(\boldsymbol{x},y)$, we instead observe i.i.d.~PL training data $\{(\boldsymbol{x}_i,s_i)\}^n_{i=1}$ from a corrupted version $p(\boldsymbol{x},s)$ of $p(\boldsymbol{x},y)$ over $\X\times\S$.
The distribution $p(\boldsymbol{x},s)$ is such that the marginal distribution of instances $p(\boldsymbol{x})$ is unchanged, but the observed label is corrupted to an ambiguous candidate-label set.
PLL tries to nonetheless learn the optimal classifier by fitting $\{(\boldsymbol{x}_i,s_i)\}_{i=1}^n$.

The key assumption in reliable PLL is that the PLs are noise-free, which means the latent true label $y_i$ of an instance $\boldsymbol{x}_i$ is always included in its candidate-label set $s_i$, i.e., 
\begin{equation*}
	p(y_i\in s_i\mid\boldsymbol{x}_i,s_i)=1,\ \forall (\boldsymbol{x}_i,y_i)\in p(\boldsymbol{x},y),\ \forall s_i\in\S.
\end{equation*}

We argue that this assumption is fairly strict since the density $p(\boldsymbol{x},y)$ of clean distribution is agnostic.
Requiring crowdsourcing workers to cautiously judge each category to ensure that the correct one must be chosen partially runs counter to the original purpose of reducing labeling costs. 
As the acquisition of training data expands, it is pervasive for label information to be corrupted, but unfortunately, it has never been considered in previous PLL works.
Thus we introduce a more general data setting titled \emph{unreliable PLL}: 
\begin{definition}[unreliable PLL]\label{def:upl}
	Given the joint density $p(\boldsymbol{x},y,s)$ and its marginal density $p(\boldsymbol{x},s)$, for any noisy PL data $(\boldsymbol{x}_i,s_i)$ independently sampled from $p(\boldsymbol{x},s)$, its true label $y_i$ has a probability of $0\le\gamma\le 1$ not being included in the candidate-label set $s_i$, i.e.,
	\begin{equation*}
		p(y_i\in s_i\mid\boldsymbol{x}_i,s_i)=1-\gamma,\ \forall (\boldsymbol{x}_i,y_i)\sim p(\boldsymbol{x},y),\ \forall s_i\in\S,
	\end{equation*}
	where $\gamma$ is called the \emph{unreliability rate}. Learning from noisy PL data is called unreliable PLL.
\end{definition}

\subsection{Robustness}\label{subsec:robust}

\quad The $\ell$-\emph{risk} of $f$ in fully-supervised learning w.r.t.~multi-class loss $\ell:\RR^k\times\Y\rightarrow\RR^+$ is defined as follows:
\begin{equation*}
	\R(f;\ell)=\EE_{p(\boldsymbol{x},y)}[\ell(f(\boldsymbol{x}),y)].
\end{equation*}
$\EE$ denotes the expectation and its subscript indicates the distribution with respect to which the expectation is taken. 
Typically, $\ell$ is \emph{classification-calibrated} \cite{bartlett2006convexity}, that is, the global minimizer of $\R(f;\ell)$ is the same as that of $\R(f;\ell_{01})$. 
$\ell_{01}$ is the zero-one loss is defined by $\ell_{01}(y,z)=\II(y=z)$ where $\II$ is the indicator function.
The \emph{Bayes optimal classifier} that minimizes $\R(f;\ell)$ is given by $f^*=\argmin_f\R(f;\ell)$, where the optimality is defined over all measurable functions.
We denote by $\R^*\doteq\R(f^*)$ the corresponding Bayes risk under the clean distribution.

Denote by $\tilde{\ell}:\RR^k\times\S\rightarrow\RR^+$ a suitably modified $\ell$ for use with PLs (defined in Section~\ref{subsec:apl}).
Similarly, the \emph{PLL risk} under $p(\boldsymbol{x},s)$ w.r.t.~\emph{PLL loss} $\tilde{\ell}$ is defined as
\begin{equation*}
	\tilde{\R}(f;\tilde{\ell})=\EE_{p(\boldsymbol{x},s)}[\tilde{\ell}(f(\boldsymbol{x}),s)].
\end{equation*}
The aim of PLL is to predict the true label for unseen instances.
However, most of the standard learning methods are hard to perform well as they tend to exhibit overfitting on the candidate labels in such scenarios \cite{lv2020progressive}.

Constructing \emph{robust} losses from the perspective of the objective function is a powerful means in weakly supervised learning \cite{manwani2013noise,ghosh2015making,menon2019can}.
Its focus is to derive the theoretical guarantee for robust losses so that the learned classifier based on weak supervision approximates the Bayes optimal classifier.
Concretely, a loss $\tilde{\ell}$ is robust to PLs (more specifically the risk minimization with $\tilde{\ell}$ is asymptotically robust to PLs) if it guarantees that the \emph{optimal PLL classifier} $\tilde{f}^*=\argmin_f\tilde{\R}(f;\tilde{\ell})$ converges to the Bayes optimal classifier.
\begin{definition}[PL-robustness]\label{def:robust}
	We say that a loss $\tilde{\ell}$ is robust to PL data (PL-robust) if for any $p(\boldsymbol{x},y)\in\X\times\Y$, $\R(\tilde{f}^*)-\R^*$ is bounded.
\end{definition}
$\R(\tilde{f}^*)-\R^*$ is bounded means that $\tilde{f}^*$ learned from PL data has a similar classification error to $f^*$ on the supervised data, i.e., minimizing $\tilde{\R}$ yields an approximate solution that minimizes $\R$.
A guarantee of robustness thus sets an analogous calibration theory \cite{bartlett2006convexity} of PLL.
Let $\tilde{\R}^*\doteq\R(\tilde{f}^*)$.
Then the robustness condition will often be rewritten as that $\tilde{\R}(f^*)-\tilde{\R}^*$ is bounded, which is slightly weak because it only signifies that $f^*$ is the approximated minimizer of $\tilde{\R}(f^*)$, but does not guarantee the classification performance of $\tilde{f}^*$ on the supervised data.

In statistical learning theory, \emph{consistency} \cite{mohri2018foundations} is another important concept.
We use the superscript $^{\bigstar}$ to indicate the optimal solution over a given hypothesis class $\F$, i.e., $f^{\bigstar}=\argmin_{f\in\F}\R(f;\ell),\ \tilde{f}^{\bigstar}=\argmin_{f\in\F}\tilde{\R}(f;\tilde{\ell})$.
Suppose $\hat{\tilde{f}}=\argmin_{f\in\F}\frac{1}{n}\sum_{i=1}^n\tilde{\ell}(f(\boldsymbol{x}_i),s_i)$ is the PLL empirical risk minimizer.
The quality of $\hat{\tilde{f}}$ with respect to $f^{\bigstar}$ is measured by the \emph{estimation error}:
\begin{equation}\label{eq:estimation}
	\R(\hat{\tilde{f}})-\R(f^{\bigstar})=\underbrace{(\R(\hat{\tilde{f}})-\R(\tilde{f}^{\bigstar})}_{\text{RHS1}}+\underbrace{(\R(\tilde{f}^{\bigstar})-\R(f^{\bigstar}))}_{\text{RHS2}}.
\end{equation}
If as $n\rightarrow\infty$, there is $\R(\hat{\tilde{f}})\rightarrow \R(f^{\bigstar})$, we say the PLL is consistent.
According to the universal approximation theorem \cite{cybenko1989approximation,anthony2009neural} that using a proper DNN, the hypothesis space $\F$ is sufficiently complex to contain the Bayes optimal classifier, we have $f^*=f^{\bigstar},\ \tilde{f}^*=\tilde{f}^{\bigstar}$.
Thanks to this, the concepts of robustness and consistency can be well connected in deep learning: RHS2 in Equation~(\ref{eq:estimation}) is just the robustness measure.
Therefore, consistency is a sufficient but not a necessary condition of robustness.
Although robustness is a weaker property than consistency, its advantage lies in no need to design an ad-hoc loss for each specific data generation process, which is generally required in the consistent methods.
In conclusion, robustness is a common and critical theoretical guarantee in supervised learning, but the mechanism by which it might be achieved remains barely understood in PLL.
To the best of our knowledge, this is the first work to analyze the robustness of PLL.

\begin{table*}
	\caption{Bounds of multi-class losses, including the mean absolute error (MAE) loss, the mean square error (MSE) loss, the reverse cross entropy (RCE) loss \cite{wang2019symmetric}, the generalized cross entropy (GCE) loss, the partially Huberised cross entropy (PCE) loss \cite{menon2019can}, the categorical cross entropy (CCE) loss, and the focal loss (FL) \cite{lin2017focal}.}
	\label{tab:loss}
	\centering
	\renewcommand\arraystretch{2}
	\begin{Large}
		\resizebox{\linewidth}{!}{
			\begin{tabular}{ll|c|c}
				\toprule
				\multicolumn{2}{c|}{Loss function} & Bound of loss & Bound of the sum of losses over all classes \\
				\midrule
				\multicolumn{1}{l|}{MAE}           & $\ell(f(\boldsymbol{x}),i)=||\boldsymbol{e}^i-f(\boldsymbol{x})||_1$ & $0\leq\ell(f(\boldsymbol{x}),i)\leq 2$ & $\sum_{i=1}^k\ell(f(\boldsymbol{x}),i)=2k-2$ \\
				\multicolumn{1}{l|}{MSE}           & $\ell(f(\boldsymbol{x}),i)=||\boldsymbol{e}^i-f(\boldsymbol{x})||_2^2$ & $0\leq\ell(f(\boldsymbol{x}),i)\leq 2$ & $k-1\leq\sum_{i=1}^k\ell(f(\boldsymbol{x}),i)\leq 2k-2$ \\
				\multicolumn{1}{l|}{RCE}& $\ell(f(\boldsymbol{x}),i)=-\sum_{j=1}^k g_j(\boldsymbol{x})\log \boldsymbol{e}^i_j$ & $0\leq\ell(f(\boldsymbol{x}),i)\leq -A, A<0$ & $\sum_{i=1}^k\ell(f(\boldsymbol{x}),i)=A-Ak$  \\
				\multicolumn{1}{l|}{GCE}           & $\ell(f(\boldsymbol{x}),i)=\frac{1-g_i(\boldsymbol{x})^q}{q}$ & $0\leq\ell(f(\boldsymbol{x}),i)\leq \frac{1}{q},\ q\in(0,1]$ & $\frac{k-k^{1-q}}{q}\leq\sum_{i=1}^k\ell(f(\boldsymbol{x}),i)\leq\frac{k-1}{q}$ \\
				\multicolumn{1}{l|}{\multirow{2}{*}{PCE}}     & \multirow{2}{*}{$\ell(f(\boldsymbol{x}),i)= \begin{cases}
						-\tau g_i(\boldsymbol{x})+\log\tau+1, &\text{if\ }g_i(\boldsymbol{x})\leq\frac{1}{\tau},\\
						-\log g_i(\boldsymbol{x}) &\text{otherwise}	
					\end{cases}$} & \multirow{2}{*}{$0\leq\ell(f(\boldsymbol{x}),i)\leq\log\tau+1, \tau>1$} & $k\log k\leq\sum_{i=1}^k\ell(f(\boldsymbol{x}),i)\leq(k-1)(\log\tau+1), \text{if\ }k\leq\tau$ \\
				\multicolumn{1}{l|}{} & & & $k-\tau+k\log\tau\leq\sum_{i=1}^k\ell(f(\boldsymbol{x}),i)\leq(k-1)(\log\tau+1), \text{if\ }k>\tau$ \\
				\multicolumn{1}{l|}{CCE}  & $\ell(f(\boldsymbol{x}),i)=-\log g_i(\boldsymbol{x})$ & $\ell(f(\boldsymbol{x}),i)\geq 0$ & Unbounded\\
				\multicolumn{1}{l|}{FL}  & $\ell(f(\boldsymbol{x}),i)=-(1-g_i(\boldsymbol{x}))^\tau\log g_i(\boldsymbol{x})$ & $\ell(f(\boldsymbol{x}),i)>0,\ \tau>0$ & Unbounded\\     
				\bottomrule                   
		\end{tabular}}
	\end{Large}
\end{table*}

\section{Methodology}\label{sec:method}

In this section, we propose a family of APL loss functions for PLs, and introduce five data generation processes.

\subsection{A Family of Average PL (APL) Losses}\label{subsec:apl}

In this paper, we propose a family of loss functions named the average PL (APL) losses following the principled ABS:
\begin{equation}\label{eq:apl}
	\tilde{\ell}(f(\boldsymbol{x}),s)=\frac{1}{|s|}\sum\nolimits_{i\in s}\ell(f(\boldsymbol{x}),i),
\end{equation}
where $|\cdot|$ represents the cardinality.
Our learning formulation is built on a simple scheme that combines multiple multi-class losses on the individual candidate. 
For example, we can use the GCE or CCE loss as the component $\ell$.
If $s$ is a singleton, the APL loss reduces to the ordinary multi-class loss. 
The idea of the APL losses comes from a practically motivated process proposed by Feng et al. \cite{feng2020provably}: they assumed that a candidate-label set is feature-independent and uniformly drawn given a specific true label, i.e., $p(s|y,\boldsymbol{x})=p(s|y)=\mathrm{const}.$ if $y\in s$, and $p(s|y)=0$ otherwise.
The generation process of noise-free PLs can thus be formalized as $p(s|\boldsymbol{x})=\sum_y p(s|y)p(y|\boldsymbol{x})\propto\sum_{y\in s}p(y|\boldsymbol{x})$. 
Then we could consider replacing the posterior with a loss and obtain $\tilde{\ell}(f(\boldsymbol{x}),s)\propto\sum_{y\in s}\ell(f(\boldsymbol{x}),y)$.
It inspires the formula of the APL losses, and the normalization term $1/|s|$ breaks the bias to training data with more candidate labels. 
The APL losses encourage the larger outputs on candidate labels, while do not explicitly guarantee the true label has the biggest score.
Ideally, a``nice'' loss $\ell$ can drive up the output of the true label implicitly resorting to the inductive bias, while a ``bad'' loss results in an inability to disambiguate.
Thus, the issue now is that which multi-class loss functions can bound $\R(\tilde{f}^*)-\R^*$ (or $\tilde{\R}(f^*)-\tilde{\R}^*$), that is, make our ABS method with the APL loss PL-robust.

Let us give an motivating example.
$\{z_1,z_2\}\in s$ are two candidate labels of an instance $\boldsymbol{x}$, and $z_1$ is true.
Then the APL loss of $f$ on this sample is $\tilde{\ell}(f(x),s)=\frac{1}{2}[\ell(f(x),z_1)+\ell(f(x),z_2)]$.
We would like to increase $g_{z_1}(\boldsymbol{x})$ to get it close to 1 so that $g_{z_2}(\boldsymbol{x})$ is decreased, signifying $f$ successfully remembers the true label without the interference of $z_2$.
Paradoxically, because all candidate labels contribute to minimizing $\tilde{\ell}$, neither $\ell(f(x),z_1)$ nor $\ell(f(x),z_2)$ should be too large.
Intuitively, if $\ell$ has an upper bound, then the value of $\tilde{\ell}$ is acceptable even if $g_{z_2}(\boldsymbol{x})$ is close to 0.
But if this is not the case, the optimization algorithm must keep $g_{z_2}(\boldsymbol{x})$ not too small to ensure that $\tilde{\ell}(f(x),z_2)$ is not too small, then memory for the true labels is hindered.
The empirical observations in Section~\ref{sec:intro} also confirm this inference.

We investigate a series of non-negative multi-class loss functions and prove that the APL losses with \emph{bounded} multi-class loss functions \footnote{Notice that if the marginal density $p(\boldsymbol{x})$ is compactly supported, given $\sup_{f\in\F}||f||_\infty\leq C_f$ where $C_f>0$ and $\F$ is a chosen function class, $\ell(f(\boldsymbol{x}),y)$ with any surrogate loss function $\ell$ is bounded. In this paper, ``bounded'' refers to the property of the loss function itself without regard to specific data distributions, function classes, or regularizations.} are robust to (both noise-free and noisy) PL data in Section~\ref{sec:result}.
\begin{definition}\label{def:bounded}
	We say a multi-class loss function is bounded if for any classifier $f$ and any input $\boldsymbol{x}\in\X$, it satisfies, for a constant $U$,
	\begin{equation}
		0\leq\ell(f(\boldsymbol{x}),i)\leq U, \ \forall i\in\Y,
	\end{equation}
	such that the sum of losses over all classes is also bounded by some constants $C_1$ and $C_2$:
	\begin{equation}\label{eq:sumbounded}
		C_1\leq\sum\nolimits_{i=1}^k\ell(f(\boldsymbol{x}),i)\leq C_2.
	\end{equation}
	Specially, if $C_1=C_2=C$, i.e., $\sum_{i=1}^k\ell(f(\boldsymbol{x}),i)=C$, the loss function is said to be symmetric.
\end{definition}

We examine widely-used loss functions and list their boundness in Table~\ref{tab:loss}. 
We use a one-hot representation for each label, i.e., if the label $y=i$, its label vector is represented as $\boldsymbol{e}^i$, where the $j$-th element is given by $\boldsymbol{e}^i_j=1$ if $i=j$, otherwise 0. 
Then for symmetric losses, $\ell(f(\boldsymbol{x}),j)=C/(k-1), \forall j\neq i$. 


\subsection{Data Generation Processes}
\begin{figure*}[!tp]
	\centering
	\includegraphics[width=0.8\linewidth]{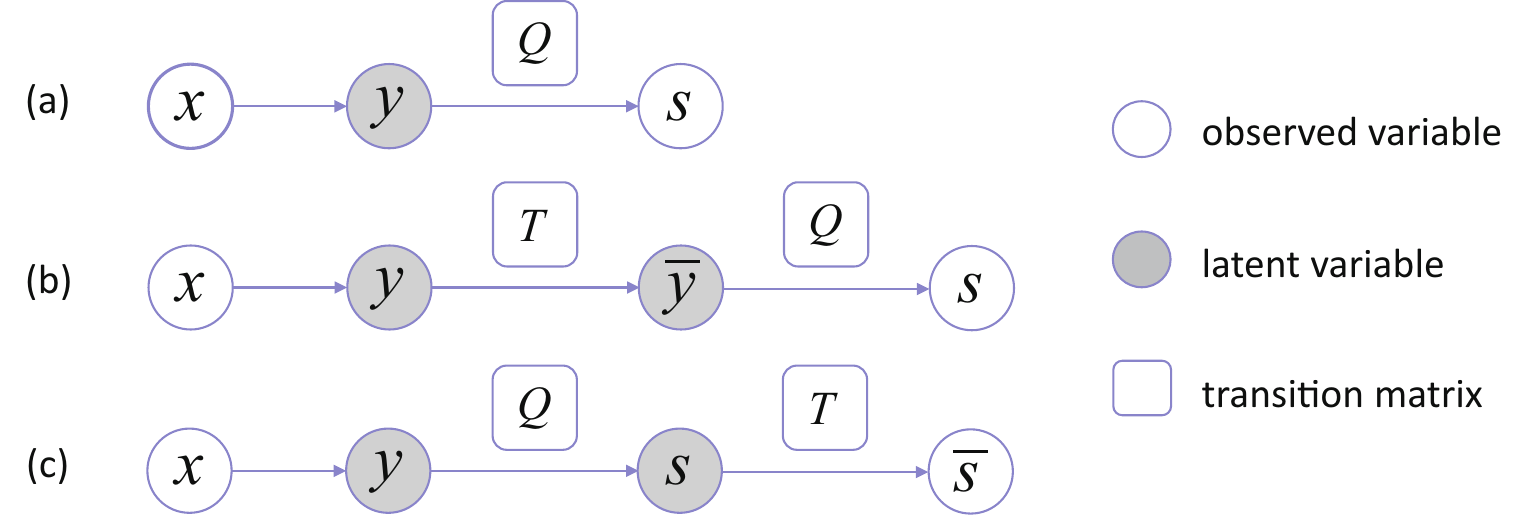} 
	\caption{Generation processes of PLs. The Filtered sampling process, Flipping process, and Global sampling process are all shown by (a), where the difference lies in the change of the PL transition matrix $\boldsymbol{Q}$. $\boldsymbol{Q}$ in (b) the Confusing process and (c) the Destructing process are constructed by either the Filtered sampling process or the Flipping process in (a). Apart from $\boldsymbol{Q}$, (b) involves the noise transition matrix $\boldsymbol{T}$ and (c) involves the complement transition matrix $\boldsymbol{T}$.} 
	\label{fig:generation}
\end{figure*}

To provide the main insights, we have to propose some assumptions on the data generation processes.
In the following, we formulate five problem settings for the generation processes of noise-free and noisy PLs which are illustrated in Figure~\ref{fig:generation}.
The first two models characterize noise-free PL data and the other three are for noisy PL data.
We follow the assumption in prior reliable PLL works \cite{liu2012conditional,feng2020provably,wen2021leveraged} and the classical feature-independent model of label noise \cite{natarajan2013learning,menon2015learning,patrini2017making,han2018co,xia2019anchor} that the observation is conditionally independent of input given the true label, as a result there are $s\perp \boldsymbol{x}\mid y, p(s|\boldsymbol{x},y)=p(s|y)$. 
Then the density of corrupted distribution is formulated as $p(\boldsymbol{x},s)=\sum_{y\in\Y}p(s|y)p(\boldsymbol{x},y)$.

\vspace{1em}                         
\noindent
\textbf{Filtered sampling process for noise-free PLs}\quad 
In a pioneering study involving PLL generation processes, the PL is assumed to be independently and uniformly sampled given a specific true label \cite{feng2020provably}.
It is inspired by a real-world cost-saving application of labeling: without any prior knowledge, the labeling system generates a random PL for each sample and asks the human annotators whether the set contains the true label.
Resampling PLs for samples for which the annotators answered ``NO''.
While it would be easy to make the labeling system have some rudimentary knowledge, for example, ``salmon'' and ``spacecraft'' do not usually appear in the same set.
Then we generalize the uniformly sampling assumption and propose the \emph{Filtered sampling process}.
Formally speaking, given a specific true label, a noise-free PL is assumed to be sampled as a whole:
\begin{equation}\label{eq:sample1}
	p(s|y) = \begin{cases}
		\eta^y_s &\text{if}\ y\in s,\\
		0 &\text{if}\ y\notin s,	
	\end{cases}
\end{equation}
where $0\le\eta^y_s\le 1$ is the \emph{sampling probability} of the label set $s$ given the true label $y$, and $\sum_{y\in s}\eta^y_s=1$. 
In particular, if the sampling probability is uniform, namely, all incorrect labels have the same probability of appearing in the label set, mathematically represented as $\eta^y_s=1/(2^{k-1}-1),\ \forall s\ni y$ here, we call this special case the \emph{Uniform filtered sampling process}.

This generation model can be written in the form of a transition matrix \cite{han2018co}.
We enumerate all the label sets $s$ in the PL space $\S$ and specify an index $l$ for each set, i.e., $l_i\in\S\ (i\in [2^k-2])$.
By this notation, we summarize all the probabilities into a \emph{PL transition matrix} $\boldsymbol{Q}^{k\times(2^k-2)}$, where $Q_{ij}=p(s=l_j|y=i)$. 
Further taking into account the assumed data distribution in Equation~(\ref{eq:sample1}), we can instantiate $\boldsymbol{Q}$ as $Q_{ij}=\eta^i_{l_j}$ if $i\in l_j$, otherwise $Q_{ij}=0$. 
Then for all $j\in [2^k-2]$, there is $p(s=l_j|\boldsymbol{x})=\sum_{i\in l_j}p(s=l_j|y=i)p(y=i|\boldsymbol{x})$.
Thus we have
\begin{equation}\label{eq:sample2}
	p(\boldsymbol{x},s)=\boldsymbol{Q}^\top p(y|\boldsymbol{x})p(\boldsymbol{x}),
\end{equation}
where $^\top$ denotes the transpose.

\vspace{1em}                         
\noindent
\textbf{Flipping process for noise-free PLs}\quad 
In the real world, however, there are complex and varying correlations between categories, making it common to get some similar categories mixed up, so that some combinations of labels appear more frequently than others.
The probability of an incorrect label appearing in the PL of a sample depends on how similar it is to its true label.
We therefore propose the \emph{Flipping process} to model this more general scenario, where a noise-free PL is supposed to be generated by adding each label to the candidate-label set independently:
\begin{equation}\label{eq:flip1}
	p(s|y)=M\prod\nolimits_{i\in s}\eta^y_i\prod\nolimits_{i\notin s}(1-\eta^y_i),
\end{equation} 
where
\begin{equation*}
	\eta^y_i=p(i\in s|y), \forall i\in\Y,\ \ \ \ M=1/\big(1-\prod\nolimits_{i\neq y}\eta^y_i\big).
\end{equation*} 
$\eta^y_i$ is the \emph{flipping probability} that depicts the probability of $i$-label being included into the candidate-label set given the specific class label $y$, and $\eta^y_y=1$. 
For $i\neq y$, it satisfies $0\le\eta^y_i<1$. 
$M$ excludes the set whose cardinality equals $k$ by re-sampling. 
Likewise, we consider a special case where the flipping probability is uniform, i.e., $\eta^y_i=\eta$, and we name it the \emph{Uniform flipping process}.

Similarly, the PL transition matrix can be formulated as $\boldsymbol{Q}^{k\times k}$ where $Q_{ij}=p(j\in s|y=i)=\eta^i_j$ and the diagonal elements of $\boldsymbol{Q}$ are all 1.
If $q(\boldsymbol{x})$ is a $k$-dimension vector where the $j$-th element $q_j(\boldsymbol{x})$ is the probability $p(j\in s|\boldsymbol{x})$, then 
\begin{equation}\label{eq:flip2}
	q(\boldsymbol{x})=\boldsymbol{Q}^\top p(y|\boldsymbol{x}),\ \ \ \ p(\boldsymbol{x},s)=\prod\nolimits_{i\in s}q_i(\boldsymbol{x})p(\boldsymbol{x}).
\end{equation}

\vspace{1em}                         
\noindent
\textbf{Global sampling process for noisy PLs}\quad 
Recall the generation scenario of the sampling process, which requires manual filtering of non-conforming label sets and resampling.
Thus noisy PLs may happen when human annotators are unprofessional.
In this way, all elements in the PL space have the probability of being sampled: 
\begin{equation}
	p(s|y) = \eta^y_s, \ \forall s\in\S,
\end{equation}
where $0\le\eta^y_s\le 1$.
Specially, the \emph{Uniform global sampling process} refers to the case where the sampling probability for all noise-free PLs is $(1-\gamma)/(2^{k-1}-1)$, and that for all noisy PLs equals $\gamma/(2^{k-1}-1)$.
In addition, the density $p(\boldsymbol{x},s)$ takes the same form as Equation~(\ref{eq:sample2}) while even if $i\notin s^j$, $Q_{ij}$ may be larger than 0.

In the following, we model two types of generation processes of noisy PLs in terms of how noise-free PLs are contaminated, where the true class is obfuscated by another similar class, or noise-free PLs are deliberately corrupted, respectively.

\vspace{1em}                         
\noindent
\textbf{Confusing process for noisy PLs}\quad 
In this type of setting, the true class was (accidentally) confused with other (similar) classes, leading to the misuse of an incorrect label as the original true label in the PL generation process.
Therefore, the \emph{Confusing process} consists of two steps.

First, the true label is corrupted.
Suppose the class-conditional label noise (CCN) model \cite{natarajan2013learning,menon2015learning,patrini2017making,han2018co,xia2019anchor}---the most widely-used model for noisy label classification---is applied, where each instance from class $y$ has a fixed probability of being assigned to label $i$, that is
\begin{equation}\label{eq:noisylabel}
	\bar{y} = \begin{cases}
		y &\text{with probability}\ 1-\gamma^y,\\
		i, i\in\Y, i\neq y &\text{with probability}\ \bar{\gamma}^y_i,	
	\end{cases}
\end{equation}
where $0\le\gamma^y\le 1$ is the \emph{label noise rate}. 
The noise is said to be uniform if $\gamma_y=\gamma$ and $\bar{\gamma}^y_i=\gamma/(k-1)$, otherwise it is said to be asymmetric. 
The corrupting step can be formalized by the \emph{noise transition matrix} $\boldsymbol{T}$ \cite{patrini2017making}, where $T_{ij}=p(\bar{y}=j|y=i)$.
Second, the corrupted label $\bar{y}$ serves as the true label $y$ to generate candidate labels, which signifies we require noisy PLs to contain $\bar{y}$, in the same way that noise-free PLs must contain $y$.
At this point, candidate labels are generated according to the previously proposed Filtered sampling process or Flipping process.
Accordingly, the following equation holds:
\begin{equation}\label{eq:confusing}
	p(\boldsymbol{x},s)=p(\boldsymbol{x})\sum_{y\in\Y}\boldsymbol{T}^{-\top} p(\bar{y}|\boldsymbol{x})p(s|y),
\end{equation}
where $p(s|y)$ can be expanded into the form of Equation~(\ref{eq:sample1}) or Equation~(\ref{eq:flip1}).

\vspace{1em}                         
\noindent
\textbf{Destructing process for noisy PLs}\quad 
We believe that noise-free PLs can also be (intentionally) destructed.
For example, there may exist spammers who deliberately choose label sets that are totally irrelevant to the tasks.
Hence we propose the \emph{Destructing process} that also contains two steps. 

Noise-free PLs are first generated by the Filtered sampling process or Flipping process, followed by taking its complement with the \emph{set flipping rate} $\gamma_s$: 
\begin{equation}\label{eq:binarynoisylabel}
	s = \begin{cases}
		s &\text{with probability}\ 1-\gamma_s,\\
		\overline{s} &\text{with probability}\ \gamma_s,
	\end{cases}
\end{equation}
where $0\le\gamma_s\le 1$. 
Constructing the \emph{complement transition matrix} $\boldsymbol{T}^{(2^k-2)\times(2^k-2)}$, each row (column) of which represents a fixed label set in the PL space.
$T_{ij}=\gamma_s$ if the $i$-th and the $j$-th label set are complementary to each other, $T_{ii}=1-\gamma_s$ for all $i\in [2^k-2]$, and 0 otherwise.
Then the density function of the PL distribution $p(\boldsymbol{x},s)$ is multiplied by $\boldsymbol{T}$ on the basis of Equation~(\ref{eq:sample2}) or Equation~(\ref{eq:flip2}).

\section{Theoretical Results}\label{sec:result}

The studied generation processes of PLs allow us to theoretically understand the properties of the APL losses. 
In this section, we detail sufficient conditions under which multi-class loss functions make PLL with the APL losses robust in various scenarios, and conclude some instructive findings.

\subsection{Robustness to Noise-Free PLs}

\begin{theorem}\label{thm:sampling}
	With the Filtered sampling process, suppose $\R^*=0$ and $\forall i\neq y,\ \sum_{s:i\in s}\eta^y_s<1$, then
	
	(1) for any symmetric loss, $\tilde{f}^*=f^*$; 
	
	(2) for any bounded loss, $0\leq\tilde{\R}(f^*)-\tilde{\R}^*\leq A(C_2-C_1)$, where $A=\sum_{i=1}^{k-1}\frac{1}{i}\sum_{s:|s|=i}\eta^y_s$.
\end{theorem}
$A$ is the \emph{weighted sum} of $p(y\in s|y)$. 
Theorem~\ref{thm:sampling} shows that under certain conditions, the risk of the Bayes classifier on PL data approaches the minimal PLL risk w.r.t.~bounded losses.
The tighter the bound of the bounded loss is, the more robust the APL loss is, and the extreme case is achieved by the symmetric loss which leads to statistical consistency (refer to Section~\ref{subsec:robust}).
The key condition for PL-robustness lies in the sampling probability.
Since $\sum_{y\in s}\eta^y_s=1$, the condition $\sum_{s:i\in s}\eta^y_s<1, \ \forall i\neq y$ indicates that any labels other than the true one are not necessarily included in the candidate-label set, i.e., \emph{the domination of true labels}.
There is another constraint $\R^*=0$ that means the classes are separable in fully-supervised classification if the multi-class loss $\ell$ is classification-calibrated.
Note that experimental results later show that even if this constraint is not satisfied, bounded losses still show good empirical PL-robustness.
We can see that the conditions for robustness do not depend on the clean distribution.

Theorem~\ref{thm:sampling} immediately leads to a special uniform case.

\begin{corollary}\label{cor:sampling}
	With the Uniform filtered sampling process, for any symmetric loss, $\tilde{f}^*=f^*$; for any bounded loss, $0\leq\R(\tilde{f}^*)-\R^*\leq\frac{A'(C_2-C_1)}{A-A'}$, where $A=\frac{1}{2^{k-1}-1}\sum_{j=1}^{k-1}\frac{1}{j}\tbinom{k-1}{j-1}$ and $A'=\frac{1}{2^{k-1}-1}\sum_{j=2}^{k-1}\frac{1}{j}\tbinom{k-2}{j-2}$. 
\end{corollary}
Here $A$ also denotes the weighted sum of $p(y\in s|y)$, and $A'$ is the weighted sum of $p(i\in s,i\neq y|y)$, representing the weighted sum of the probabilities that an incorrect label appears in the candidate-label set given the specific true label.
The uniform sampling probability has already ensured the domination of true labels, and the separability constraint $\R^*=0$ is eliminated, which implies that even in the stochastic scenario, learning with the APL losses can be PL-robust.
Therefore, in this case, the robustness is satisfied without any constraint.
Noticed that in the uniform case, the optimal PLL classifier approaches the Bayes classifier, so that the PL-robustness is better guaranteed.

Similarly, we derive the PL-robustness conditions considering the Flipping process.

\begin{theorem}\label{thm:flipping}
	With the Flipping process, suppose $\R^*=0$, then
	
	(1) for any symmetric loss, $\tilde{f}^*=f^*$;
	
	(2) for any bounded loss, $0\leq\tilde{\R}(f^*)-\tilde{\R}^*\leq MA(C_2-C_1)$, where $A=\prod_{i\neq y}(1-\eta_{yi})+\frac{1}{2}\sum_{i=1,i\neq y}^k\eta_{yi}\prod_{j\neq y,i}(1-\eta_{yj})+\frac{1}{3}\sum_{i=1,i\neq y}^{k-1}\sum_{j=i+1,j\neq y}^{k}\eta_{yi}\eta_{yj}\prod_{m\neq y,i,j}(1-\eta_{ym})+\ldots+\frac{1}{k-1}\sum_{i=1,i\neq y}^k(1-\eta_{yi})\prod_{j\neq y,i}\eta_{yj}$.
\end{theorem}

\begin{corollary}\label{cor:flipping}
	With the Uniform flipping process, for any symmetric loss,  $\tilde{f}^*=f^*$; for any bounded loss, $0\leq\R(\tilde{f}^*)-\R^*\leq \frac{A'(C_2-C_1)}{A-A'}$, where $A=\sum_{j=1}^{k-1}\frac{1}{j}\tbinom{k-1}{j-1}\eta^{j-1}(1-\eta)^{k-j}$ and $A'=\sum_{j=2}^{k-1}\frac{1}{j}\tbinom{k-2}{j-2}\eta^{j-1}(1-\eta)^{k-j}$.
\end{corollary}

Comparing Theorem~\ref{thm:flipping} and Corollary~\ref{cor:flipping} versus Theorem~\ref{thm:sampling} and Corollary~\ref{cor:sampling}, we can notice that their sufficient conditions and upper bounds of the difference in the risk are different in notations but similar in meaning.
There is no constraint on the flipping probability because the diagonal-dominance of the PL transition matrix has pledged the domination of true labels.

\subsection{Robustness to Noisy PLs}

Then we discuss the robustness to noisy PLs under three generation processes. 
For the first one we give formal statements, and for the latter two, we summarize the result in Theorem~\ref{thm:ambiguousnoise}, giving an intuitive statement on the conditions, and more formal mathematical details are deferred in Appendix A.

\begin{theorem}\label{thm:arbitrary}
	With the Global sampling process, suppose $\R^*=0$ and the domination relations hold: $d(y)>d(j)\ \forall j\neq y$ where $d(\cdot)$ is defined as
	$d(i)=\sum_{j=1}^{k-1}\frac{1}{j}\sum_{s:|s|=j,i\in s}\eta_{ys}, \forall i\in\Y$, then
	
	(1) for any symmetric loss, $\tilde{f}^*=f^*$;
	
	(2) for any bounded loss, $0\leq\tilde{\R}(f^*)-\tilde{\R}^*\leq d(y)(C_2-C_1)$.
\end{theorem}

\begin{corollary}\label{cor:arbitrary}
	With the Uniform global sampling process, if $\gamma<1/2$, then for any symmetric loss, $\tilde{f}^*=f^*$; for any bounded loss, $0\leq\R(\tilde{f}^*)-\R^*\leq \frac{A'(C_2-C_1)}{A-A'}$, where $A=\frac{1-\gamma}{2^{k-1}-1}\sum_{j=1}^{k-1}\frac{1}{j}\tbinom{k-1}{j-1}$ and $A'=\frac{1-\gamma}{2^{k-1}-1}\sum_{j=2}^{k-1}\frac{1}{j}\tbinom{k-2}{j-2}+\frac{\gamma}{2^{k-1}-1}\sum_{j=1}^{k-1}\frac{1}{j}\tbinom{k-2}{j-1}$.
\end{corollary}

\begin{figure*}[!t]
	\begin{minipage}{.04\textwidth}
		\centering
		\begin{tabular}{l}
			MNIST
		\end{tabular}
	\end{minipage}
	\hfill
	\begin{minipage}{.26\textwidth}
		\centering
		\includegraphics[width=4cm]{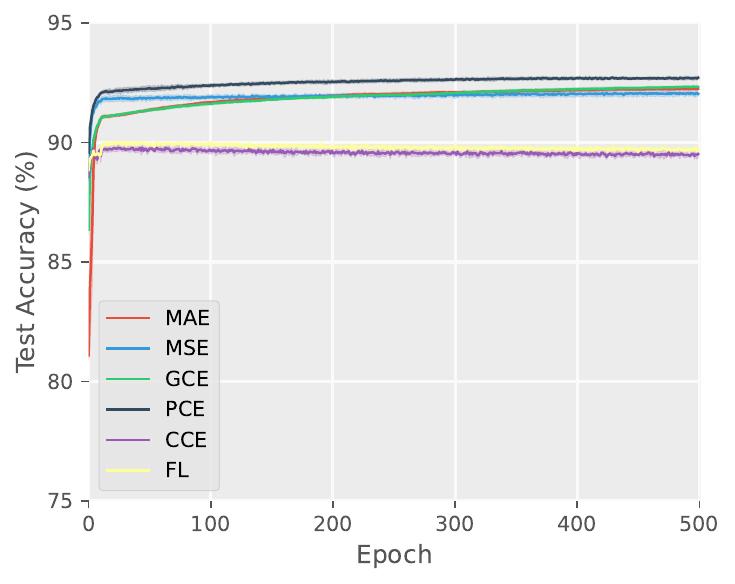}
		\centerline{\quad Linear with NF-PLs}
	\end{minipage}\hspace{-8mm} 
	\begin{minipage}{.26\textwidth}
		\centering
		\includegraphics[width=4cm]{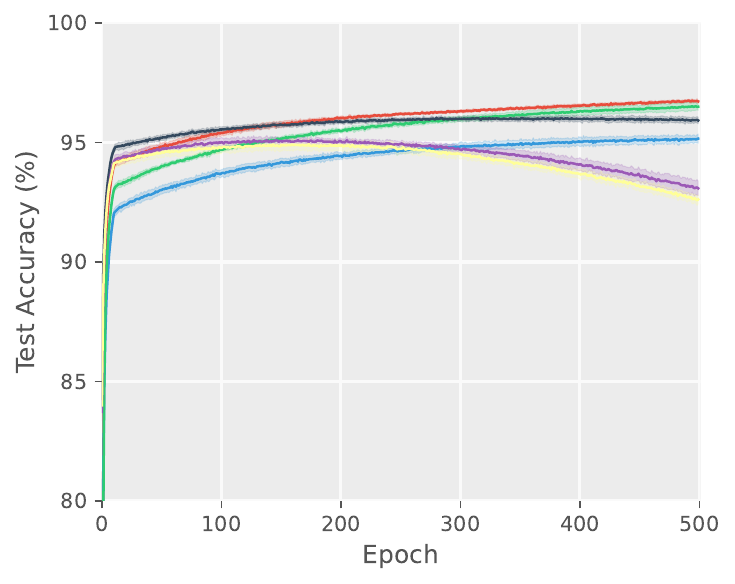}
		\centerline{\quad MLP with NF-PLs}
	\end{minipage}\hspace{-8mm}
	\begin{minipage}{.26\textwidth}
		\centering
		\includegraphics[width=4cm]{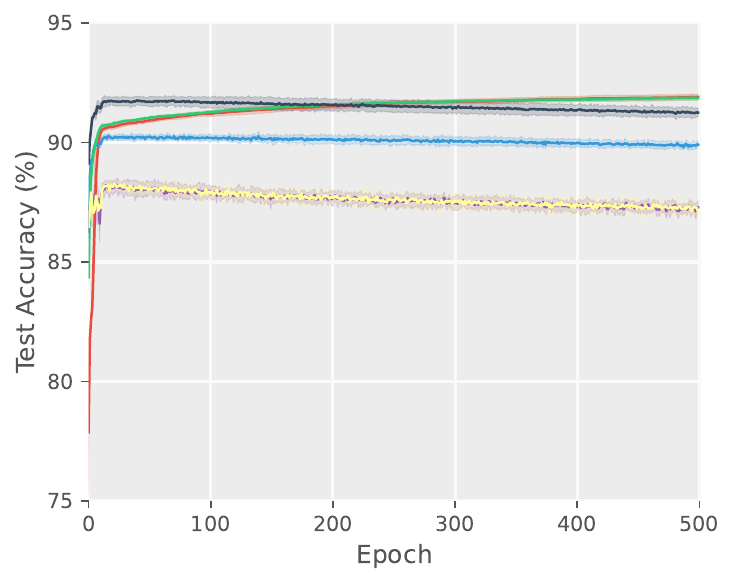}
		\centerline{\quad Linear with N-PLs}
	\end{minipage}\hspace{-8mm}
	\begin{minipage}{.26\textwidth}
		\centering
		\includegraphics[width=4cm]{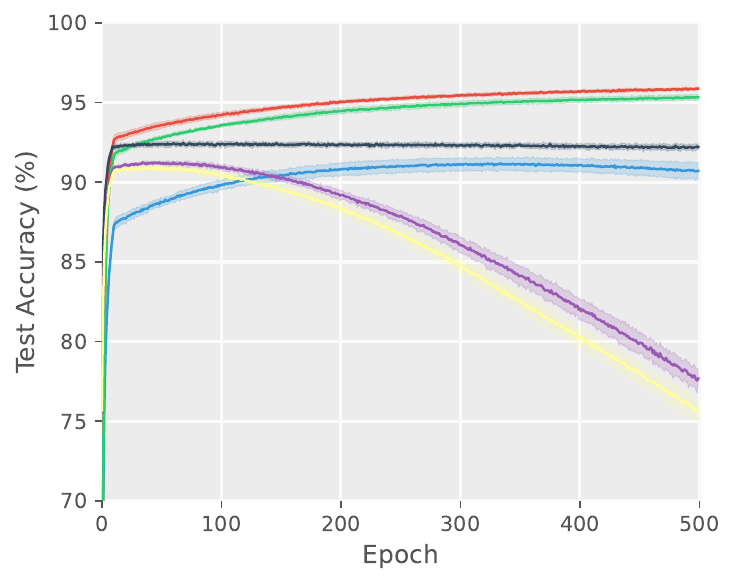}
		\centerline{\quad MLP with N-PLs}
	\end{minipage}
	
	\begin{minipage}{.04\textwidth}
		\centering
		\begin{tabular}{l}
			CIFAR-\\10
		\end{tabular}
	\end{minipage}
	\hfill
	\begin{minipage}{.26\textwidth}
		\centering
		\includegraphics[width=4cm]{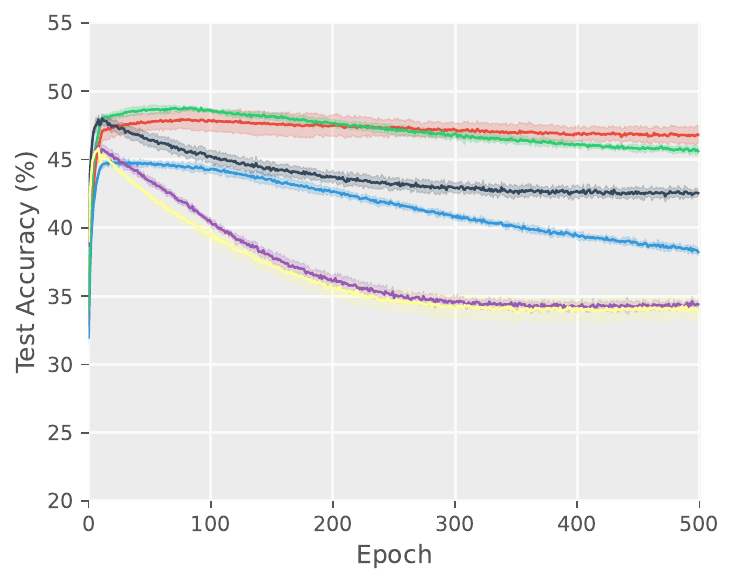}
		\centerline{\quad MLP with NF-PLs}
	\end{minipage}\hspace{-9mm} 
	\begin{minipage}{.26\textwidth}
		\centering
		\includegraphics[width=4cm]{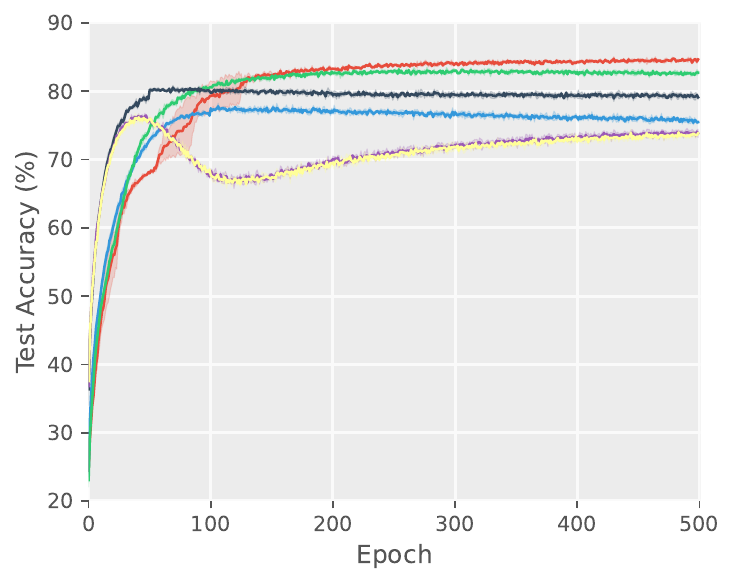}
		\centerline{\quad ConvNet with NF-PLs}
	\end{minipage}\hspace{-9mm} 
	\begin{minipage}{.26\textwidth}
		\centering
		\includegraphics[width=4cm]{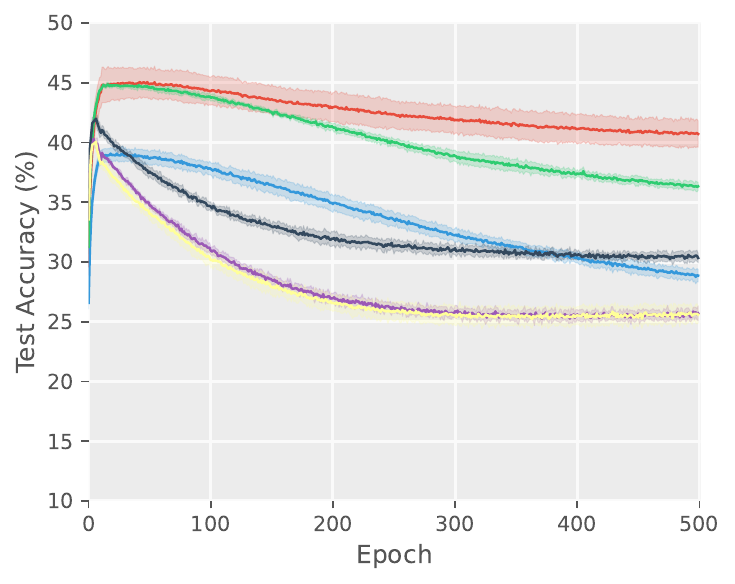}
		\centerline{\quad MLP with N-PLs}
	\end{minipage}\hspace{-9mm} 
	\begin{minipage}{.26\textwidth}
		\centering
		\includegraphics[width=4cm]{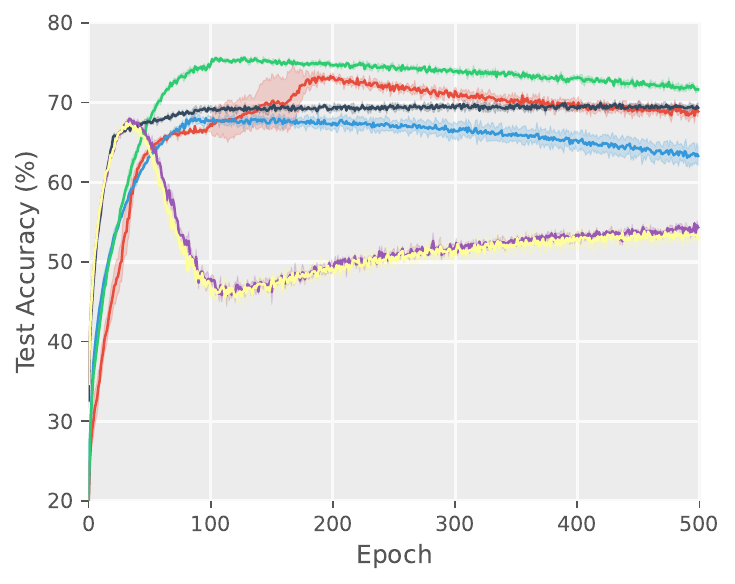}
		\centerline{\quad ConvNet with N-PLs}
	\end{minipage}
	\caption{Test accuracy of our ABS method with bounded versus unbounded losses on benchmarks.}
	\label{fig:testacc}
\end{figure*}

In Theorem~\ref{thm:arbitrary}, the domination relations mean that the weighted sum of $p(y\in s|y)$ is larger than that of $p(i\in s|y)$ for any incorrect label $i\neq y$.
Compared with noise-free PLs, the constraint on the dominance relationship in the condition for robustness to noisy PLs is tightened from the sum of probabilities to the weighted sum of probabilities.
Besides, we find that with the uniform sampling probability, the constraint is fairly loose: the probability of outputting noise-free PLs is greater than fifty percent, requiring only a little more domain knowledge than a completely random labeling system.

As the Confusing process and Destructing process encompass a variety of cases, we first give a sweeping summary by the following theorem, and then discuss each in detail.

\begin{theorem}\label{thm:ambiguousnoise}
	With the Confusing process or the Destructing process, suppose $\R^*=0$ and the domination relations hold: the weighted sum of $p(y\in s|y)$ is always larger than that of $p(i\in s|y), \forall i\neq y$ , then
	
	(1) for any symmetric loss, $\tilde{f}^*=f^*$;
	
	(2) for any bounded loss, $0\leq\tilde{\R}(f^*)-\tilde{\R}^*\leq A(C_2-C_1)$, where $A$ is a constant associated with $p(y\in s|\boldsymbol{x})$.
\end{theorem}

First we investigate the Confusing process.
We simplify the process by assuming the uniform cases in the generation of candidate labels.
This simplification degenerates the domination relations to $\bar{\gamma}^y_i<1-\gamma^y,\ \forall i\neq y$.
We were a little surprised to find that it becomes identical to the condition for asymmetric-noise-tolerance (Theorem 3 in \cite{ghosh2017robust}). 
Further supposing that the true labels are corrupted uniformly, we have $\gamma<(k-1)/k$.
This is the same as the noise-tolerant condition under uniform noise (Theorem 1 in \cite{ghosh2017robust}).
Moreover, the constraint $\R^*=0$ is removed and $\R(\tilde{f}^*)-\R^*$ is bounded in this case.
These facts demonstrate that, as long as the candidate labels are generated uniformly, the robustness condition to noisy PLs is \emph{exclusively determined} by the label noise rate.

Next we probe into the Destructing process through similar renderings. 
When the candidate labels are generated in a uniform manner, the domination relations are reduced to $\gamma_{s}<1/k,\ \forall s$.
If the probability of every candidate-label set being destructed is also uniform, the set flipping rate is also the unreliability rate, namely $\gamma_{s}=\gamma,\ \forall s$. 
Then this process is essentially the same as the Uniform global process.
We again show that in the case of a uniform generation process for candidate labels, the PL-robustness condition is only related to the extent to which the candidate-label set is unreliable.

\vspace{1em}  
\noindent
\textbf{Remarks}\quad 
Through the theorems above, we can summarize some of their commonalities:
\begin{itemize}[topsep=0ex,itemsep=0ex,leftmargin=*]
	\item The critical condition that makes the APL losses PL-robust is consistent in all generation scenarios: the weighted sum of the probabilities that the true label is associated with the instance dominates;
	\item For noisy PL data, if the candidate labels are generated uniformly, the robustness condition is completely determined by the degree of unreliability rate;
	\item For bounded losses, the tighter the bound is, the stronger PL-robustness the APL losses have, and the most desirable situation (classifier-consistency) is achieved by the symmetric loss.
\end{itemize}

The above theoretical findings provide guidance to the design of losses of ABS.
Note that IBS is heuristic and not really ERM-based (refer to Section.~\ref{sec:enhance}), and on this account, the robustness of IBS could hardly be proven.

\begin{figure*}[!t]
	\begin{minipage}{.04\textwidth}
		\centering
		\begin{tabular}{l}
			MNIST
		\end{tabular}
	\end{minipage}
	\hfill
	\begin{minipage}{.26\textwidth}
		\centering
		\includegraphics[width=4cm]{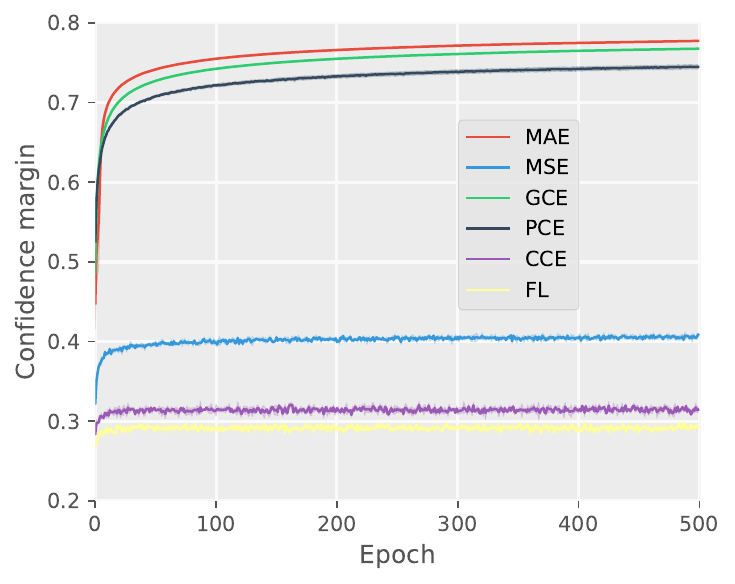}
		\centerline{\quad Linear with NF-PLs}
	\end{minipage}\hspace{-8mm} 
	\begin{minipage}{.26\textwidth}
		\centering
		\includegraphics[width=4cm]{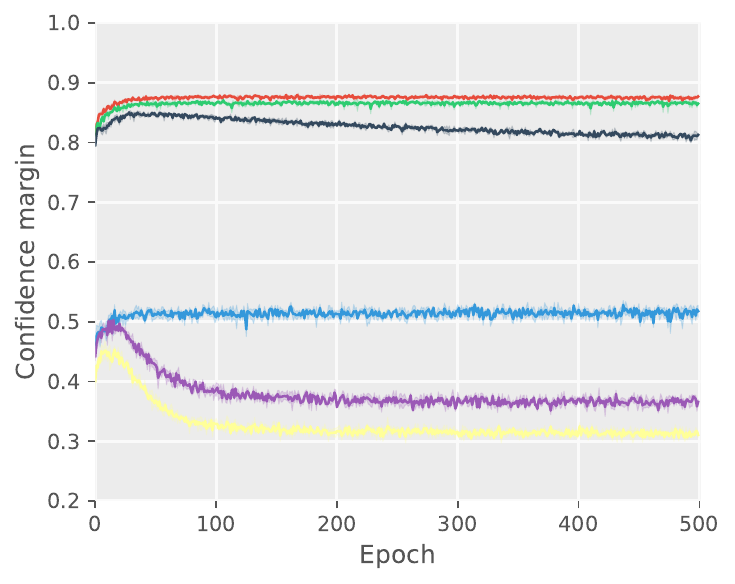}
		\centerline{\quad MLP with NF-PLs}
	\end{minipage}\hspace{-8mm}
	\begin{minipage}{.26\textwidth}
		\centering
		\includegraphics[width=4cm]{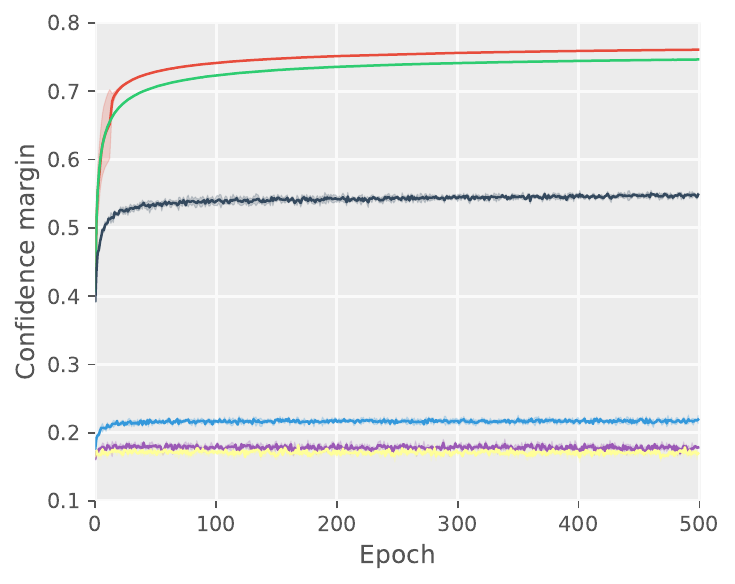}
		\centerline{\quad Linear with N-PLs}
	\end{minipage}\hspace{-8mm}
	\begin{minipage}{.26\textwidth}
		\centering
		\includegraphics[width=4cm]{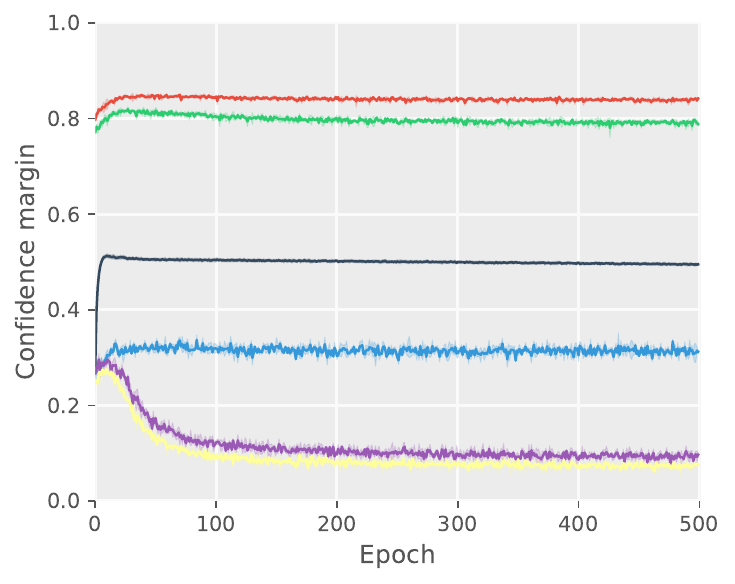}
		\centerline{\quad MLP with N-PLs}
	\end{minipage}
	
	\begin{minipage}{.04\textwidth}
		\centering
		\begin{tabular}{l}
			CIFAR-\\10
		\end{tabular}
	\end{minipage}
	\hfill
	\begin{minipage}{.26\textwidth}
		\centering
		\includegraphics[width=4cm]{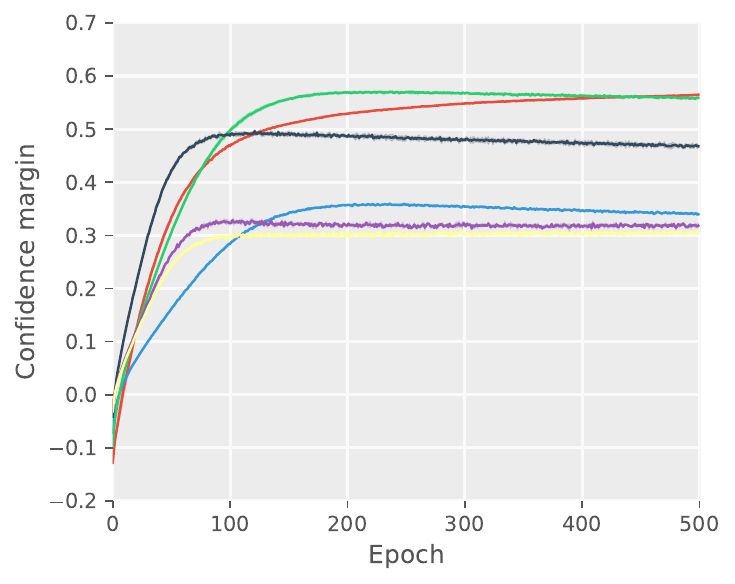}
		\centerline{\quad MLP with NF-PLs}
	\end{minipage}\hspace{-9mm} 
	\begin{minipage}{.26\textwidth}
		\centering
		\includegraphics[width=4cm]{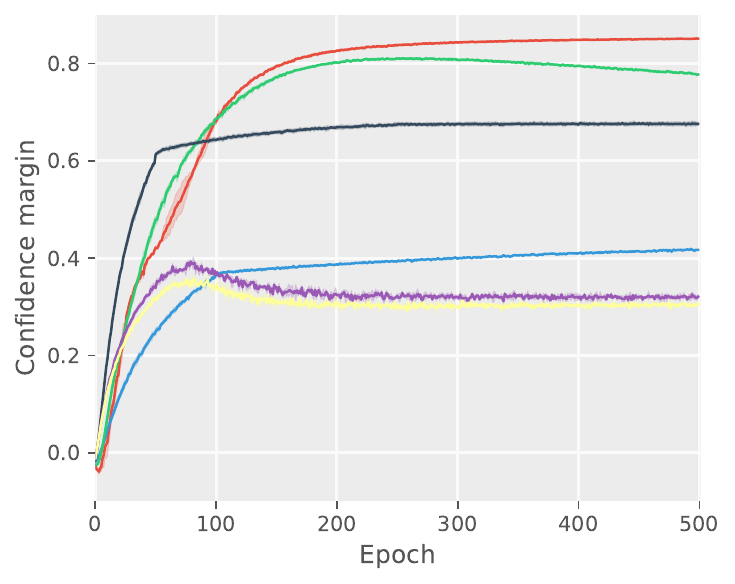}
		\centerline{\quad ConvNet with NF-PLs}
	\end{minipage}\hspace{-9mm} 
	\begin{minipage}{.26\textwidth}
		\centering
		\includegraphics[width=4cm]{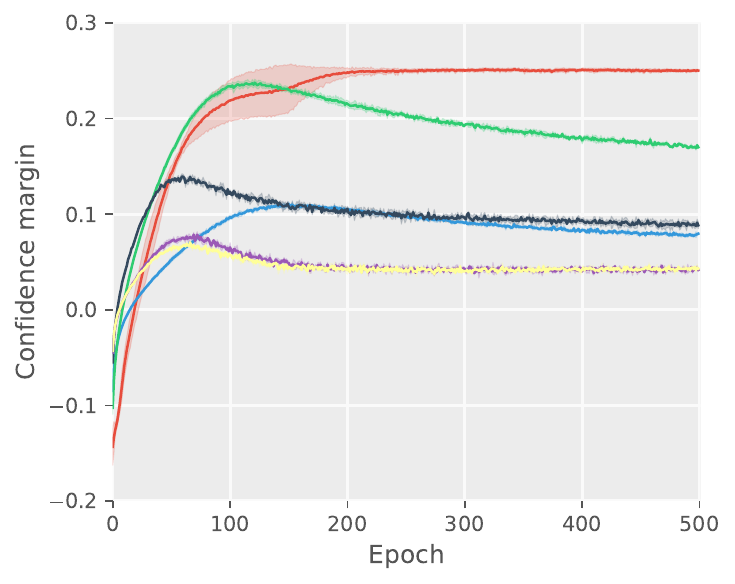}
		\centerline{\quad MLP with N-PLs}
	\end{minipage}\hspace{-9mm} 
	\begin{minipage}{.26\textwidth}
		\centering
		\includegraphics[width=4cm]{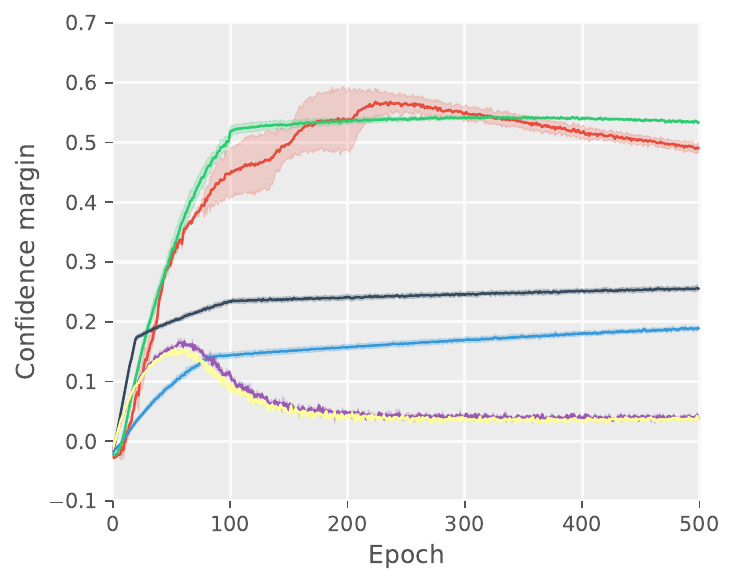}
		\centerline{\quad ConvNet with N-PLs}	
	\end{minipage}
	\caption{Confidence margin of our ABS method with bounded versus unbounded losses on benchmarks.}
	\label{fig:margin}
\end{figure*}

\begin{table*}[!t]
	\centering
	\caption{Means$\pm$standard deviations of test accuracy in percentage with different data generations.}
	\label{tab:benchmark}
	\renewcommand\arraystretch{1.25}
	\setlength{\tabcolsep}{0.6mm}{
		\begin{tabular}{l|lc|cccc|cc}
			\toprule
			\multirow{2}{*}{Dataset} & \multirow{2}{*}{Model} & \multirow{2}{*}{Case} & \multicolumn{4}{c|}{Bounded} & \multicolumn{2}{c}{Unbounded} \\
			&  &  & MAE & MSE & GCE & PCE & CCE & FL \\\midrule
			\multirow{8}{*}{MNIST}           & \multirow{4}{*}{Linear} & 1 & $\bm{91.17\pm0.07}$ & 87.61$\pm$0.07 & 90.40$\pm$0.12  & 86.28$\pm$0.40 & 85.13$\pm$0.29 & 85.10$\pm$0.24  \\
			&         & 2 & 92.25$\pm$0.08 & 92.03$\pm$0.10 & 92.33$\pm$0.06 & $\bm{92.71\pm0.08}$ & 89.51$\pm$0.17 & 89.71$\pm$0.13 \\
			&         & 3 & $\bm{91.90\pm0.13}$ & 89.90$\pm$0.10 & 91.86$\pm$0.10 & 91.25$\pm$0.19 & 87.28$\pm$0.27 & 87.08$\pm$0.27 \\
			&         & 4 & $\bm{90.37\pm0.19}$ & 82.70$\pm$0.49 & 85.87$\pm$0.48 & 80.61$\pm$0.41 & 67.78$\pm$1.19 & 67.63$\pm$0.84 \\\cline{2-9}
			& \multirow{4}{*}{MLP-5}   & 1 & $\bm{94.44\pm0.19}$ & 82.28$\pm$0.83 & 92.77$\pm$0.21 & 87.62$\pm$0.27 & 84.80$\pm$0.50 & 84.15$\pm$0.73 \\
			&         & 2 & $\bm{96.71\pm0.05}$ & 95.16$\pm$0.19 & 96.49$\pm$0.15 & 95.92$\pm$0.13 & 93.08$\pm$0.30 & 92.62$\pm$0.18 \\
			&         & 3 & $\bm{95.87\pm0.85}$ & 80.71$\pm$0.52 & 95.34$\pm$0.19 & 91.14$\pm$0.25 & 77.68$\pm$0.53 & 75.66$\pm$0.74 \\
			&         & 4 & $\bm{93.18\pm0.47}$ & 89.12$\pm$1.03 & 89.19$\pm$0.47 & 86.52$\pm$0.82 & 77.57$\pm$1.32 & 76.45$\pm$0.91 \\\midrule
			\multirow{8}{*}{\begin{tabular}[c]{@{}l@{}}Fashion-\\ MNIST\end{tabular}}   & \multirow{4}{*}{Linear}   & 1 & 81.50$\pm$1.82 & 80.23$\pm$0.29 & $\bm{82.50\pm0.21}$ & 79.85$\pm$0.38 & 77.92$\pm$0.49 & 77.56$\pm$0.58 \\
			&         & 2 & 81.08$\pm$0.02 & 83.59$\pm$0.24 & $\bm{84.72\pm0.12}$ & 84.35$\pm$0.12 & 81.88$\pm$0.36 & 81.91$\pm$0.09 \\
			&         & 3 & 83.55$\pm$2.42 & 81.12$\pm$0.30 & $\bm{84.25\pm0.09}$ & 81.60$\pm$0.28 & 79.36$\pm$0.24 & 79.37$\pm$0.31 \\
			&         & 4 & 76.40$\pm$0.50 & 73.87$\pm$0.37 & $\bm{78.52\pm0.47}$ & 61.50$\pm$0.67 & 43.83$\pm$0.38 & 40.39$\pm$0.93 \\\cline{2-9}
			& \multirow{4}{*}{MLP}   & 1 & $\bm{86.37\pm0.64}$ & 81.26$\pm$0.47 & 84.29$\pm$0.10 & 76.01$\pm$0.49 & 51.29$\pm$0.83 & 49.36$\pm$0.48 \\
			&         & 2 & $\bm{87.52\pm0.09}$ & 84.49$\pm$0.27 & 87.32$\pm$0.31 & 85.47$\pm$0.05 & 74.10$\pm$0.64  & 73.50$\pm$0.24 \\
			&         & 3 & $\bm{85.74\pm0.27}$ & 81.62$\pm$0.33 & 84.42$\pm$0.37 & 80.33$\pm$0.51 & 53.09$\pm$0.74 & 52.75$\pm$0.81 \\
			&         & 4 & $\bm{82.41\pm0.25}$ & 73.49$\pm$0.53 & 66.83$\pm$0.87 & 65.73$\pm$0.87 & 25.76$\pm$0.38 & 22.01$\pm$0.53 \\\midrule
			\multirow{8}{*}{\begin{tabular}[c]{@{}l@{}}Kuzushiji-\\ MNIST\end{tabular}}   & \multirow{4}{*}{Linear}& 1 & 63.14$\pm$0.14 & 59.65$\pm$0.62 & $\bm{65.58\pm0.32}$ & 58.78$\pm$0.46 & 54.90$\pm$0.61 & 55.02$\pm$0.48 \\
			&         & 2 & 64.35$\pm$0.09 & $\bm{68.54\pm0.31}$ & 65.04$\pm$0.20 & 67.40$\pm$1.93 & 63.79$\pm$0.34 & 63.74$\pm$0.44 \\
			&         & 3 & 63.55$\pm$0.22 & 65.26$\pm$0.21 & 64.24$\pm$0.29 & $\bm{67.75\pm0.17}$ & 59.82$\pm$0.36 & 60.16$\pm$0.73 \\
			&         & 4 & $\bm{60.34\pm0.67}$ & 56.16$\pm$0.80 & 51.98$\pm$0.44 & 51.47$\pm$0.80 & 36.48$\pm$0.59 & 35.76$\pm$0.31 \\\cline{2-9}
			& \multirow{4}{*}{MLP}    & 1 & $\bm{81.72\pm0.50}$ & 64.60$\pm$1.06  & 75.71$\pm$0.15 & 57.17$\pm$0.22 & 36.07$\pm$0.16 & 26.13$\pm$0.46 \\
			&         & 2 & $\bm{86.99\pm0.25}$ & 74.24$\pm$0.13 & 86.59$\pm$0.24 & 80.90$\pm$0.41 & 66.06$\pm$0.51 & 64.81$\pm$0.41 \\
			&         & 3 & 78.86$\pm$2.25 & 69.90$\pm$0.61 & $\bm{79.26\pm0.54}$ & 72.89$\pm$0.18 & 56.43$\pm$0.88 & 54.96$\pm$0.64 \\
			&         & 4 & $\bm{66.61\pm0.29}$ & 54.93$\pm$0.18 & 43.06$\pm$0.44 & 42.58$\pm$0.43 & 23.93$\pm$0.44 & 23.35$\pm$0.63 \\\midrule
			\multirow{8}{*}{\begin{tabular}[c]{@{}l@{}}CIFAR-\\ 10\end{tabular}}   & \multirow{4}{*}{MLP}& 1 & $\bm{41.46\pm0.72}$ & 33.11$\pm$0.69 & 31.98$\pm$0.40 & 31.10$\pm$0.42 & 22.47$\pm$0.49 & 22.75$\pm$0.28 \\
			&         & 2 & $\bm{48.69\pm0.25}$ & 38.99$\pm$0.42 & 46.00$\pm$0.30 & 42.94$\pm$0.44 & 36.81$\pm$0.22 & 36.37$\pm$0.35 \\
			&         & 3 & $\bm{40.83\pm0.48}$ & 29.05$\pm$0.49 & 35.30$\pm$0.30 & 30.90$\pm$0.49 & 28.04$\pm$0.34 & 27.31$\pm$0.18 \\
			&         & 4 & $\bm{31.62\pm0.60}$ & 25.52$\pm$0.55 & 20.75$\pm$0.35 & 21.10$\pm$0.37 & 13.51$\pm$0.34 & 13.70$\pm$0.35 \\\cline{2-9}
			& \multirow{4}{*}{ConvNet}   & 1 & 66.86$\pm$1.63 & 61.67$\pm$0.31 & $\bm{68.65\pm0.71}$ & 34.93$\pm$0.59 & 33.67$\pm$0.79 & 34.17$\pm$0.89 \\
			&         & 2 & $\bm{86.65\pm0.13}$  & 75.47$\pm$0.11  & 82.74$\pm$0.29 & 79.11$\pm$0.32 & 72.82$\pm$0.61 & 73.83$\pm$0.21 \\
			&         & 3 & 68.88$\pm$0.65 & 63.30$\pm$1.16  & $\bm{71.65\pm0.27}$ & 69.35$\pm$0.65 & 53.13$\pm$0.63 & 52.56$\pm$1.14 \\
			&         & 4 & $\bm{51.96\pm4.21}$ & 49.58$\pm$0.62 & 43.83$\pm$0.36 & 43.51$\pm$0.61 & 17.29$\pm$0.30 & 17.02$\pm$0.14 \\    \bottomrule       
	\end{tabular}}
\end{table*}

\subsection{Estimation Error Bound}

Let us review the relation between robustness and consistency again.
Now we have proved the condition that bounds RHS2 in Equation~(\ref{eq:estimation}) (assuming $\F$ is instantiated to be a DNN).
Then we establish the estimation error bound and show that as the number of training data approaches infinity, RHS1 is also bounded.

Suppose $\G_y$ be a class of real functions, and $\F=\oplus_{y\in[k]}\G_y$ be a $k$-valued function class.
Assume there are $C_{f}>0$ and $C_{\ell}>0$ such that $\sup_{f\in\F}||f||_\infty\leq C_{f}$ and  $\sup_{x\in\X,f\in\F,y\in\Y}\ell(f(x),y)\leq C_{\ell}$, and assume $\ell(f(x),y)$ is $\rho$-Lipschitz continuous for all $||f||_\infty\leq C_{f}$.
The \emph{Rademacher complexity} of $\G_y$ over $p(x)$ with sample size $n$ is defined as $\Ra_n(\G_y)$ \cite{bartlett2002rademacher,mohri2018foundations}. 
Then we have the following estimation error bound.
\begin{theorem}\label{thm:errorbound}
	For any $\delta>0$, we have with probability at least $1-\delta$,
	\begin{align}\label{eq:errorbound}
		\tilde{\R}(\hat{\tilde{f}})- \tilde{\R}(\tilde{f}^{\bigstar}) \leq 4\sqrt{2}k\rho\sum_{y=1}^k \Ra_n(\G_y) + 2C_{\ell}\sqrt{\frac{\log (2/\delta)}{2n}}.
	\end{align}
\end{theorem}
As $n\rightarrow\infty$, $\Ra_n(\G_y)\rightarrow 0$ for all parametric models with a bounded norm such as DNNs trained with weight decay \cite{lu2018minimal}, which signifies $\tilde{\R}(\hat{\tilde{f}})\rightarrow\tilde{\R}(\tilde{f}^{\bigstar})$.

\section{Experimental Findings}\label{sec:exp}

In this section, we provide some empirical understandings of our ABS method, experimentally validating our theoretical findings on benchmark datasets, which then enlighten an improvement of IBS methods.
The implementation is based on PyTorch \cite{paszke2019pytorch} and experiments were carried out with NVIDIA Tesla V100 GPU.

\subsection{Empirical Understanding of APL Losses}\label{sec:empirical}
\textbf{Our ABS method with bounded loss functions are robust.}\quad 
We first run a set of experiments on MNIST \cite{lecun1998gradient} and CIFAR-10 \cite{krizhevsky2009learning} to verify whether our ABS method with bounded losses is robust to both noise-free PLs and noisy PLs, whereas they are not with unbounded losses.
We generate noise-free PLs by the Uniform flipping process with flipping probability 0.1, and noisy PLs by the Confusing process where the label noise is uniform and $\gamma=0.3$. 
Then the candidate labels are generated in the same way as the previous noise-free PLs.
On each dataset, we train two networks using the APL losses with different multi-class losses, e.g., bounded versus unbounded ones in Table~\ref{tab:loss}.
The RCE loss is, up to a constant of proportionality, equivalent to the MAE loss and omitted.
We set the focusing parameter 0.5 for the FL loss.
Detailed settings are in Section~\ref{sec:evaluation}.

The test accuracy with different losses is presented in Figure~\ref{fig:testacc}.
As we have theoretically proved, learning with bounded losses is robust: after reaching a peak in test accuracy, their test accuracy is relatively flat throughout the training process.
However, unbounded losses exhibit significant overfitting in most cases.
Specifically, the symmetric loss MAE has the smoothest curve, while other bounded losses would be slightly overfitting in difficult learning scenarios.
The same results are shown across different datasets, under different data settings, with different models.
In general, the more difficult the learning scenario is (e.g., harder datasets and weaker supervised information), the larger the gap between bounded loss and unbounded loss is, because unbounded losses overfit more severely.

\begin{table*}[!t]
	\centering
	\caption{PRODEN \& Our ABS method on CIFAR datasets. E and A stand for early stopping and robust warm start, respectively. The best and equivalent based on the paired $t$-test at the significance level 5\% are shown in boldface by comparing the 1st and 3rd columns, the 2nd and 4th columns. ``---'' means that we skipped the experiments under the reliable PLL setting. The best combination is \underline{underline}.}
	\label{tab:abs+ibs}
	\renewcommand\arraystretch{1.25}
	\setlength{\tabcolsep}{1.6mm}{
		\begin{tabular}{l|lrr|cc|cc|c}
			\toprule
			\multirow{2}{*}{Dataset} & \multirow{2}{*}{Model} & \multirow{2}{*}{Case} & \multicolumn{1}{l|}{E:} & \text{\sffamily X} & \checkmark & \text{\sffamily X} & \checkmark  & Our ABS \\
			&   &   & \multicolumn{1}{l|}{A:} & \text{\sffamily X} & \text{\sffamily X} & \checkmark & \checkmark & w/ E \\\midrule
			\multirow{8}{*}{CIFAR-10} 
			& \multirow{4}{*}{MLP} & 1                     &  & 48.28$\pm$0.62 & 48.33$\pm$0.54 &  \underline{$\bm{48.77\pm0.26}$} & $\bm{48.54\pm0.36}$  & --- \\
			&  & 2 & & 51.49$\pm$0.30 & 52.17$\pm$0.21 & $\bm{52.31\pm0.32}$ & \underline{$\bm{52.70\pm0.14}$} & --- \\
			& & 3 & & 38.52$\pm$0.16 & 44.27$\pm$0.13 & $\bm{39.00\pm0.38}$ & \underline{$\bm{46.67\pm0.30}$} & 45.86$\pm$0.14 \\
			& & 4 & & 29.19$\pm$0.63 & 34.04$\pm$0.56 & $\bm{29.30\pm0.43}$ & \underline{$\bm{35.12\pm0.42}$} & 34.54$\pm$0.36 \\\cline{2-9}
			& \multirow{4}{*}{ConvNet} & 1 &  & 85.38$\pm$0.11 & 85.52$\pm$0.25 & $\bm{85.93\pm0.31}$ & \underline{$\bm{86.09\pm0.20}$} & --- \\
			& & 2 & & 88.62$\pm$0.19 & 88.29$\pm$0.33 &  \underline{$\bm{89.42\pm0.17}$} & $\bm{89.05\pm0.29}$  & --- \\
			& & 3 & & 64.59$\pm$0.70 & 73.05$\pm$0.31 & $\bm{66.86\pm0.72}$ & $\bm{75.99\pm0.39}$ & \underline{76.72$\pm$0.12} \\
			& & 4 & & 47.35$\pm$0.47 & 53.47$\pm$0.38 & $\bm{50.77\pm0.82}$ & \underline{$\bm{56.68\pm0.49}$} & 55.50$\pm$0.24 \\\midrule
			\multirow{3}{*}{CIFAR-100} & \multirow{3}{*}{ConvNet} & 2 &  & 54.49$\pm$0.46 & 59.90$\pm$0.53 & $\bm{58.31\pm0.32}$ &  \underline{$\bm{60.60\pm0.22}$} & --- \\
			& & 3 & & 34.78$\pm$0.31 & 42.04$\pm$0.22 & $\bm{36.27\pm0.43}$ & $\bm{43.67\pm0.22}$ & \underline{44.75$\pm$0.15} \\
			& & 4 & & 37.37$\pm$0.32 & 47.29$\pm$0.31 & $\bm{39.55\pm0.24}$ & \underline{$\bm{47.86\pm0.20}$} & 46.13$\pm$0.24 \\\bottomrule  
	\end{tabular}}
\end{table*}

\vspace{1em}  
\noindent
\textbf{How do the models fit the candidate labels when training with the APL losses?}\quad 
As we discussed before, ABS methods are free from identifying the true labels during training.
One may wonder how ABS methods learn from PL data.
This raises the question: is the learned model able to identify the true label of the training samples?
We investigate this problem by looking at the \emph{confidence margin} between the model's output on the true label and the maximum output on the other labels, i.e., $\mathrm{confidence-margin}(x_i)=g_{y_i}(x_i)-\max_{j\neq y_i}g_{j}(x_i)$.
The larger the margin is, the greater the likelihood that the model will successfully identify the true label for the training samples is. 
In Figure~\ref{fig:margin}, we illustrate the mean confidence margin over the training set.
We find that the margins trained with bounded losses are generally much higher than those trained with unbounded losses.
This means that although our ABS method does not explicitly disambiguate the candidate-label sets during the training phase, our ABS method with bounded losses is still able to robustly fit the true labels against the interference of other candidates, thereby explaining the good prediction performance in the test set.

\subsection{Evaluation on Benchmark Datasets}\label{sec:evaluation}
\noindent
\textbf{Setup}\quad
Experiments were conducted on four widely-used benchmark datasets including MNIST, Fashion-MNIST \cite{xiao2017fashion}, Kuzushiji-MNIST \cite{clanuwat2018deep}, and CIFAR-10. 
On each dataset, we generated PLs by 
(Case 1) the Uniform filtered sampling process; 
(Case 2) the Uniform flipping process with $\eta=0.1$;
(Case 3) the Confusing process where the label noise rate equals $0.3$ and the candidate labels were generated according to Case 2;
(Case 4) the Destructing process where the candidate labels were generated according to Case 1 and the set flipping rate equals $0.3$. 
We split the original training dataset into training and validation data with a proportion of $9:1$, and added candidate labels to both of them.

We employed various base models including a linear-in-input model (Linear), a 5-layer perceptron (MLP), and a 12-layer convolutional neural network (ConvNet) \cite{huang2017densely}. 
Linear was trained on MNIST-like datasets, ConvNet was trained on CIFAR-10, and MLP was trained on all datasets. 
The optimizer was stochastic gradient descent with momentum 0.9. 
We trained each model 500 epochs with the mini-batch size set to 256, and recorded the test accuracy of the hyper-parameters (learning rate and weight decay) with the best validation accuracy.
We did not use any manual learning rate decay and early stopping.

\vspace{1em}  
\noindent
\textbf{Results}\quad
Tables~\ref{tab:benchmark} shows the test accuracy over 5 trials. 
The best and comparable methods based on the paired $t$-test at the significance level 5\% were highlighted in boldface. 
We can see that the bounded loss always outperforms the unbounded loss, especially on the complex models.
In difficult scenarios, i.e., unreliable PLL, the accuracy of the complex models trained with bounded losses is almost always better than their linear counterpart, but unbounded losses can make the complex models overfit very badly on some tasks, causing their performance to become worse.

\subsection{Enhancing IBS Methods with ABS}\label{sec:enhance}
We revisit the SOTA IBS methods PRODEN \cite{lv2020progressive}, RC \cite{feng2020provably}, and LW \cite{wen2021leveraged}.
Their typical learning objective is as follows:
\begin{equation*}
	\hat{\R}(f;\ell)=\frac{1}{n}\sum_{i=1}^n\big[\sum_{j\in s_i}w^i_j\ell(g_j(\boldsymbol{x}_i),j)+\sum_{j\notin s_i}w^i_j\ell(g_j(\boldsymbol{x}_i),j)\big],
\end{equation*}
where $w_j$ is a weight for $j\in[k]$.
The weights of the labels that are more likely to be true are progressively increased.
Generally speaking, they initialize the weights to be
\begin{equation*}
	w^i_j=1/|s_i|,\ \forall j\in s_i.
\end{equation*}
They train a learning model $f$ with the uniform weights for several epochs for a warm start, and then update $w$ and $f$ seamlessly for the remaining epochs.
We highlight that uniform weights are necessary to break the circular dependency existing between $w$ and $f$: $f$ needs to be trained with reasonable $w$ and $w$ needs to be estimated by well-trained $f$.
The success of the algorithms is built on the observation that even if each sample has multiple candidate labels, $f$ will remember the true one first \cite{arpit2017closer}.
Thus they adjust $w$ by the output of $f$.
It indicates that an IBS method has to be pretrained a little in an ABS manner.

While we note that they always use the CCE loss in both the ABS-style phase and the subsequent IBS-style phase, which could potentially select incorrect true labels at the very beginning and negatively affect the model training.
Therefore, we introduce an enhanced principle to incorporate our theoretical findings into existing IBS methods: training the learning model with \emph{a robust warm start} to avoid overfitting.
We replace their loss functions of the first 20 epochs with the APL loss with MAE, and then switch back to their original objective function.
The hyper-parameters are tuned according to the original methods.

We considered Case 1, 2, 3, and 4 for CIFAR-10, and Case 2, 3, and 4 where the candidate labels were generated by the Uniform flipping process and $\eta=0.01$ for CIFAR-100\cite{krizhevsky2009learning}.
We used the same training/validation setting, models, and optimizer as in Section~\ref{sec:evaluation}.
We summarize the results without/with early stopping of PRODEN in Table~\ref{tab:abs+ibs}, which means that we report the last epoch or the epoch in which the best validation accuracy was reached during training.
The results of RC and LW are put in Appendix B.

From Table~\ref{tab:abs+ibs}, we can see that the enhanced method with the robust warm start has significant performance improvements over its original version.
The model has better performance when early stopping is not deployed, suggesting that the robust warm start helps not to remember incorrect labels at an early stage. 
Even after using early stopping, our enhanced version also allows for further performance improvements.
Moreover, we presented the results of our ABS method with early stopping in the last column.
It is usually comparable with the highest accuracy, meaning that our proposal is a simple yet effective baseline for unreliable PLL. 
The results on CIFAR-10 are also shown in Figure~\ref{fig:ibsVSabs}.

\section{Conclusion}\label{sec:conclu}
In this paper, we rethought the forgotten ABS in the era of deep PLL, and improved it theoretically and practically.
Theoretically, we proposed five data generation processes of noise-free and noisy PLs, and analyzed the conditions that ABS is robust to PLs, which filled the theoretical gap in the robustness analysis of PLL.
Practically, we conducted extensive experiments to confirm our theoretical findings, and showed that the IBS methods could be improved from our work, which pushed forward PLL as a whole.

\section*{Acknowledgments}
BL and XG was supported by the National Key Research and Development Plan of China (No.2018AAA0100104), and the National Natural Science Foundation of China (62125602, 62076063). 
LF was supported by the National Natural Science Foundation of China (Grant No. 62106028) and CAAI-Huawei MindSpore Open Fund.
NX was supported by China Postdoctoral Science Foundation (2021M700023), Jiangsu Province Science Foundation for Youths (BK20210220). 
GN and MS were supported by JST AIP Acceleration Research Grant Number JPMJCR20U3, Japan. MS was also supported by the Institute for AI and Beyond, UTokyo. 

\bibliographystyle{ieeetr}
\bibliography{arxiv}

\end{document}